\documentclass[11pt]{article}
\usepackage{fullpage}

\date{}

\title{\bfseries\papertitle}%\\in the Plackett-Luce model} 
\author{
  Aadirupa Saha \quad\quad Akshay Krishnamurthy\\
  {\normalsize Microsoft Research, NYC}\\
  {\small\texttt{\{aadirupa.saha,akshaykr\}@microsoft.com}}
}

\usepackage[T1]{fontenc}
\usepackage{lmodern}
\usepackage[utf8]{inputenc} % allow utf-8 input
\usepackage{url}            % simple URL typesetting
\usepackage{booktabs}       % professional-quality tables
\usepackage{nicefrac}       % compact symbols for 1/2, etc.
\usepackage{microtype}      % microtypography
\usepackage{xcolor}         % colors

\usepackage[square]{natbib}
\usepackage{nameref}
\usepackage[colorlinks, linkcolor=blue, filecolor = blue, citecolor = blue, urlcolor = blue]{hyperref}
\usepackage{amsmath, amsthm, amssymb, amsfonts}
\usepackage{bm}
\usepackage{mathtools}
\usepackage{hyperref}
\usepackage{algorithm}
\usepackage{algorithmic}
\usepackage{tabularx}
\usepackage{thm-restate}
\usepackage{subfig}
\usepackage{color}
\usepackage{enumitem}
\usepackage{times}
\usepackage{color-edits}
\addauthor{as}{cyan}
\addauthor{ak}{orange}

\usepackage{mathtools}
\usepackage{thmtools}

\newtheorem{thm}{Theorem}%[section]

\newtheorem{cor}[thm]{Corollary}
\newtheorem{defn}[thm]{Definition}

\newtheorem{rem}{Remark}

\newcommand{\regsq}{{\mathrm{Reg}_{\mathrm{Sq}}}}
\newcommand{\oreg}{\texttt{SqrReg}}

\DeclareMathOperator{\Reg}{Reg}

\newcommand{\R}{{\mathbb R}}

\newcommand{\N}{{\mathbb N}}
\renewcommand{\P}{{\mathbf P}}
\newcommand{\E}{{\mathbf E}}
\newcommand{\I}{{\mathbf I}}

\newcommand{\0}{{\mathbf 0}}

\newcommand{\cK}{{\mathcal K}}

\newcommand{\cH}{{\mathcal H}}
\newcommand{\cD}{{\mathcal D}}
\newcommand{\cF}{{\mathcal F}}
\newcommand{\cX}{{\mathcal X}}
\newcommand{\cA}{{\mathcal A}}
\newcommand{\cB}{{\mathcal B}}

\newcommand{\cP}{{\mathcal P}}

\newcommand{\hy}{{\hat y}}

\newcommand{\hf}{{\hat f}}

\newcommand{\hP}{{\hat P}}

\newcommand{\C}{{\mathbf C}}

\newcommand{\cZ}{{\mathcal Z}}
\newcommand{\cY}{{\mathcal Y}}
\newcommand{\cG}{{\mathcal G}}
\renewcommand{\cH}{{\mathcal H}}
\newcommand{\M}{{\mathbf M}}

\newcommand{\U}{{\mathbf U}}
\newcommand{\Q}{{\mathbf Q}}

\newcommand{\Y}{{\mathbf Y}}

\newcommand{\e}{{\mathbf e}}

\newcommand{\p}{{\mathbf p}}
\newcommand{\q}{{\mathbf q}}

\newcommand{\w}{{\mathbf w}}
\newcommand{\x}{{\mathbf x}}

\newcommand{\stdb}{\textrm{StdDB(K)}}

\newcommand{\sm}{{\setminus}}

\newcommand{\zs}{\mathrm{ZS}}
\newcommand{\vn}{{\texttt{VN}}}

\def \wbreg {{\mathrm{BR}\mbox{-}\mathrm{Regret}(f^\star_1,\ldots,f^\star_T)}}
\def \breg {{\mathrm{BR}\mbox{-}\mathrm{Regret}}}

\def \algf {\texttt{MinMaxDB}}
\def \algc {\texttt{CCE-DB}}
\def \algl {\texttt{CCE-linDB}}
\def \papertitle {{Efficient and Optimal Algorithms for Contextual Dueling Bandits under Realizability}}

\newcommand{\red}[1]{\textcolor{red}{#1}}
\newcommand{\blue}[1]{\textcolor{blue}{#1}}
\allowdisplaybreaks

\usepackage{mkolar_definitions}

\begin{document}

\maketitle

\begin{abstract}
We study the $K$-armed contextual dueling bandit problem, a sequential
decision making setting in which the learner uses contextual
information to make two decisions, but only observes
\emph{preference-based feedback} suggesting that one decision was
better than the other. We focus on the regret minimization problem
under realizability, where the feedback is generated by a pairwise
preference matrix that is well-specified by a given function class
$\cF$. We provide a new algorithm that achieves the optimal regret
rate for a new notion of best response regret, which is a strictly
stronger performance measure than those considered in prior works. The
algorithm is also computationally efficient, running in polynomial
time assuming access to an online oracle for square loss regression
over $\cF$. This resolves an open problem of~\citet{CDB} on oracle
efficient, regret-optimal algorithms for contextual dueling bandits.

%%%%%%%%%%%%%%%%%%%%%%%%%%%%
\end{abstract}

\section{Introduction}
\label{sec:intro}

%% \akcomment{Intro outline attempt: (1) In many decision-making
%%   scenarios, a significant obstacle towards deploying reinforcement
%%   learning is the design of the reward function. In such applications,
%%   it may be beneficial to rely on implicit feedback rather than
%%   engineer a reward function that may misalign with long-term
%%   objectives. (2) Prefence/comparative feedback is a particular form
%%   of implicit feedback that is often available --- or can be made
%%   available --- in many applications. For example, ... (3) Online decision
%%   making from preference feedback is typically modeled with the
%%   ``Dueling bandits'' framework, where ... . Explain prior work and
%%   deficiencies. (4) In this paper, we study the contextual version of
%%   the dueling bandit problem which}----\ascomment{incorporated}

In many decision-making scenarios, a significant obstacle towards
deploying reinforcement learning is the design of the reward function.
For example, in personalization applications, reward engineering to
align the performance of the reinforcement learning algorithm with
application-specific objectives often requires months of iterating in
a trial-and-error manner, requiring substantial effort from domain
experts and resulting in sub-optimal system performance in the
interim. Thus instead of engineering a reward function
that may misalign with long-term objectives, it may be beneficial to
re-design the system to collect more reliable signals that enable
efficient optimization. 

Preference/comparative feedback is a particular
signal
%% form of implicit feedback 
that is often available --- or can be made available --- in many
applications, and is often more reliable than ordinal/absolute
rewards. Preference-based feedback can be easily collected in
applications including online retail chain optimization, prediction
markets, tournament ranking, recommender systems, search engine
optimization and information retrieval, robotics, multiplayer games, and elsewhere. As just one
example,~\citet{hofmann2013fidelity} interleave results from two
different search engine ranking algorithms and use click information
as a preference signal, which they show has high fidelity and is
significantly less expensive than collecting relevance judgements from
experts.

Motivated by such scenarios, recent work from the machine learning
community has studied online decision making from pairwise/preference
feedback through the \emph{dueling bandits} framework. This framework
is a variant of the widely-studied multi-armed bandit (MAB)
setting~\citep{auer02,Slivkins19,CsabaNotes18}, where rather than
receive rewards, the learner obtains preference-based feedback
information. 
%% e.g. online retail chain,
%% prediction markets, tournament ranking, recommender systems, search
%% engine optimization, training robots, designing multiplayer games and
%% many more.  Online decision making from pairwise-preference feedback
%% has been popularly modeled as ``Dueling Bandits'' framework, which is
%% a sequential decision making framework similar to the well known
%% (stochastic) multi-armed bandit (MAB)
%% problem~\citep{auer02,Slivkins19,CsabaNotes18}, however unlike MAB,
%% the former only assumes access to relative/preference based feedback:
%
%for identifying a set of `good' arms from a fixed decision-space (set of arms / items) by querying preference feedback of actively chosen item-pairs, which has gained widespread attention in the machine learning community \citep{Yue+12,Zoghi+14RCS,Zoghi+15} over the past decade. 
%
In particular, the learner repeatedly selects a pair of items to be
compared to each other in a so-called duel and observes a stochastic binary
outcome, indicating the winning item in this duel.  Performance is
often measured by some notion of regret, and while many definitions
have been studied~\citep{Yue+12,Zoghi+14RUCB,ADB}, they all
intuitively ask that the learner identify the good actions, i.e.,
those that are typically favored amongst the others.
%%  with
%% respect to some `good policies' in hindsight, depending on the
%% underlying preference relation \citep{Yue+12,Zoghi+14RUCB,ADB}.
Over the last two decades, several algorithms have been proposed for
dueling bandit problems~\citep{Ailon+14,Zoghi+14RUCB,Komiyama+15,DTS}
and generalizations to subset-wise preference
feedback~\citep{Sui+17,Brost+16,SG19,Ren+18,SGwin18}.

In practice, preferences over items can vary substantially with
auxilliary/side information like user demographics, search query, etc.;
however, the majority of dueling bandits literature does not leverage
contextual information to learn higher-quality decision making
policies. This shortcoming motivated~\citet{CDB} to formulate the
\emph{contextual dueling bandits} problem, in which the agent first
receives a context, chooses a pair of actions, and then observes the
outcome of the duel, with the goal of learning a \emph{policy} that
maps contexts to actions that typically win \emph{in that
context}. They formulated a new notion of regret and designed two types
of algorithms: (1) a regret-optimal algorithm that is computationally
intractable, and (2) a computationally tractable algorithm with
suboptimal regret. Thus, their work left open the following question:
\begin{quote}
\begin{center}
  \emph{Is there a computationally efficient and statistically optimal algorithm for contextual dueling bandits?}
\end{center}
\end{quote}
In this paper, we resolve this question in the affirmative under a
natural realizability assumption.

\subsection{Our contributions}
\label{sec:cont}
Our main contribution is a new efficient algorithm for contextual
dueling bandits. To state the guarantee, let $\Xcal$ be a context
space, let $\cA:=[K]$ be an action space of size $K$, and let $\Pcal :=
\{\P \in [-1,1]^{K \times K}: P[i,j] = -P[j,i], P[i,i] = 0\}$ denote the set of \emph{preference matrices}, which are
skew-symmetric matrices with bounded entries and $0$ along the
diagonal. We interpret a preference matrix $\P$ as encoding a zero-sum game, in which the row player's goal is to maximize their value and the column player's goal is to minimize.

In a stochastic contextual dueling bandit instance, the
learner interacts with a distribution $\Dcal$ over $\Xcal\times\Pcal$
via the following protocol: at each round $t$ (1) nature samples
$(x_t,\P_t) \sim \Dcal$ and reveals $x_t$ to the learner, (2) learner
chooses (potentially randomly) two actions $(a_t,b_t) \in [K]^2$, (3)
learner observes $o_t \sim
\textrm{Ber}(\tfrac{P_t[a_t,b_t]+1}{2})$. The goal of the learner is
to choose actions $(a_t,b_t)$ so as to minimize the \emph{best-response regret} over $T$ rounds:
\begin{align}
	\breg_T := \sum_{t=1}^T \max_{\q_t \in \Delta_K} \frac{1}{2} \EE_{a \sim \q_t} \EE_{(a_t,b_t) \sim \p_t} \sbr{f^\star(x_t)[a, a_t] + f^\star(x_t)[a,b_t]}.
	\label{eq:br_reg}
\end{align} 
Here we use $(a_t,b_t) \sim \p_t$ to capture the learner's randomness
and define $f^\star: x \mapsto \EE_{\Dcal}[\P \mid x]$ to denote the
conditional mean of the preference matrix $\P$ given context
$x$. Intuitively, achieving low regret requires that the learner's
distribution $\p_t$ cannot be exploited by an adversary that knows the
expected preference matrix $f^\star$, so that $\p_t$ is typically preferred over any other
distribution.

We consider the function approximation setting, where we are given a
function class $\Fcal: \Xcal\to \Pcal$ that we may use to learn the
conditional mean function $f^\star$. To enable this we make two
somewhat standard assumptions about $\Fcal$, which have also appeared in prior work~\citep{dylan1,dylan2,simchi20,ag12,dylan3}.

\begin{assumn}[Function approximation assumptions, informal]
%% \label{assum:fn_approx}
We assume \emph{realizability}, that is $f^\star \in \Fcal$. We also
assume access to an \emph{online square loss oracle} for $\Fcal$
whose $T$-step square loss regret w.r.t. $\Fcal$ is bounded by
a known function $\regsq(T)$.
\end{assumn}
See Section~\ref{sec:oreg} for a detailed description of the online
square loss oracle. In this setting, our main theorem is as follows.

\begin{thm}
  \label{thm:main}
  Under the above function approximation assumptions, Algorithm \textbf{\algf}$(\gamma)$, with learning rate $\gamma = O\Big(\sqrt{\frac{KT}{\regsq(T)}}\Big)$ ensures:
  \begin{align}
    \breg_T \leq O\rbr{\sqrt{KT \cdot \regsq(T)} }.
  \end{align}
  for any $T > 4K \regsq(T)$. Additionally, \textbf{\algf}\, incurs at most a $\textrm{poly}(K)$
  factor of run-time overhead over the square loss oracle.
\end{thm}

In the sequel, we list several instantiations for the online square
loss oracle, but briefly the algorithm achieves (1) the optimal
$\sqrt{KT\log |\Fcal|}$ regret for finite function classes, (2)
$\sqrt{K^3T}$ regret for the non-contextual problem, which is a factor
of $K$ worse than the optimal rate, and (3) $\sqrt{KdT}$ regret when
$\Fcal$ is (low-dimensional) linear functions.  This is the first
oracle-efficient algorithm for contextual dueling bandits with regret
scaling at the optimal $\sqrt{T}$-rate. See Corollary \ref{cor:minmax} for detailed discussions.

\paragraph{Other Contributions.} 
In addition to our main result (Theorem~\ref{thm:main}), the paper contains the following contributions:
%% Beside the main contribution (Thm. \ref{thm:genf_informal}), we also contribute the following %The main contributions of this paper can be summarized as follows:

%\begin{itemize}[leftmargin=!]
	\vspace{3pt} \noindent
	\textbf{(1).} The notion of best-response regret in Eq.~\eqref{eq:br_reg} itself is new, and we provide connections to other regret definitions in the literature. In particular, we show that it upper bounds the policy regret definition from~\citet{CDB} and also subsumes some other notions studied in dueling bandits. 
%% Under realizability, we propose a new notion of regret (called \emph{Best-Response regret}) and connect it to other prior notions of regret studied in literature; in particular, shown it to be stronger than the notion of \emph{Policy-regret} used in \cite{CDB}, which is the only other notion of regret used in the contextual dueling bandit literature. 
%% 	(see related works and the problem formulation in Sec. \ref{sec:prob}).
	
	\vspace{3pt} \noindent
	\textbf{(2).} We also make a connection between the dueling bandits literature and the literature on Markov games, which we believe was previously unexplored. In particular, the Markov games literature has used game-theoretic techniques to developed UCB-based algorithms that can be applied directly to dueling bandits problems. We elaborate on this connection and provide complete analyses for these algorithms in an effort to encourage more cross-pollination between these communities (Sec. \ref{sec:linf}).
%% Another significant contribution of this work lies in \emph{bridging the understandings between two player markov games and dueling bandits} setup both of which essentially aims to analyze the same setting, however their connection was not explored prior to this work. Precisely, we are the first to borrow the techniques from Markov games literature and device a strategy to design efficient and optimal no-regret algorithms for Best-Response regret in contextual dueling bandits under certain structured realizability classes $\cF$ (Thm. \ref{thm:stdk}, \ref{thm:linf}).

	% Our main contribution lies in analyzing the case for any general realizable function class. {Assuming and efficient online regression oracle} (Sec. \ref{sec:oreg}), we propose a computationally efficient algorithm (Alg. \ref{alg:genf}, Sec. \ref{sec:genf}) with $\tO(\sqrt{T})$ \footnote{$\tO(\cdot)$ notation hides the logarithmic dependencies.} regret guarantee (Thm. \ref{thm:minmax}, Sec. \ref{sec:genf}). In fact we show the above performance bound of our algorithm holds good for an even stronger notion of regret where $f^\star_1,\ldots,f^\star_T \in \cF$s can be chosen even adversarially per round (see Defn. \ref{defn:wbreg} and Rem. \ref{rem:minmax}).

	\vspace{3pt} \noindent
	\textbf{(3).} Finally, we provide some evidence suggesting that in the absence of realizability, significantly new techniques are required to develop oracle-efficient $\sqrt{T}$-regret algorithms for contextual dueling bandits. Our evidence does not rule out such a result altogether, but it shows that the existing techniques from the standard contextual bandits literature are insufficient. We leave developing such an algorithm as an interesting open problem (Sec. \ref{sec:example}). 
%% we also remark that without the realizability assumption,  significantly new techniques are required to develop oracle-efficient $O(\sqrt{T})$ algorithms for agnostic contextual dueling bandits setup (Sec. \ref{sec:example}). We leave this as a fascinating open problem.
	% The main observation is that all such algorithms for standard contextual bandits establish some concentration inequality on the regret of all policies $\pi \in \Pi$, but establishing such a guarantee in the dueling setting requires incurring large	regret.
%\end{itemize}

%\red{@Akshay: Any remark on implication of our results?}

\subsection{Related work}
\label{sec:rel}
%The problem of regret minimization for stochastic multiarmed bandits (MAB) is very well studied in the online learning literature \citep{auer02,TS12,CsabaNotes18,Audibert+10,Kalyanakrishnan+12}, where the learner gets to see a noisy draw of absolute reward feedback of an arm upon playing a single arm per round. 
%
%On the other hand over the last decade, the relative feedback variants of stochastic MAB problem has seen a widespread resurgence in the form of the Dueling Bandit problem, where, instead of getting a noisy feedback of the reward of the chosen arm, the learner only gets to see a noisy feedback on the pairwise preference of two arms selected by the learner. The objective of the learner is minimize the regret with respect to `best-arm' in the stochastic model. Several algorithms has been proposed to address this dueling bandits problem, for different notions of `best arms' or preference models \citep{Busa_mallows,Busa_pl,Zoghi+14RCS,Zoghi+14RUCB,Zoghi+15,Komiyama+15,DTS,CDB}, or even extending the pairwise preference to subsetwise preferences  \citep{Sui+17,Brost+16,SG18,SGwin18,Ren+18}.

\paragraph{Non-contextual dueling bandits.}
Our work builds on a large body of literature on the non-contextual
(stochastic) dueling bandits problem, which can be seen as a special
case of our setup where there is only a single context, $|\Xcal|=1$,
and hence a single preference matrix $\P$. For the non-contextual
problem, the dominant algorithmic strategy is based on optimism in the
face of uncertainty, which is widely deployed across sequential
decision making. In terms of results, various regret definitions,
largely motivated by social choice theory, have been studied. The most
frequently used benchmark is the \emph{Condorcet winner}, which is an
arm $a^\star$ that beats all others on
average~\citep{Yue+12,Zoghi+14RUCB,Zoghi+15MRUCB,Komiyama+15,BTM}. Generalizing
slightly, one can consider a \emph{Fixed-Benchmark} regret:
\begin{align}
	\label{eq:stdb_reg}
\textrm{FB-Regret}_T:= \E_{a^\star \sim \q^\star}\bigg[\sum_{t=1}^T\E_{(a_t,b_t)\sim\p_t}\frac{P[a^\star,a_t] + P[a^\star,b_t]}{2}\bigg],
\end{align}
where $q^\star \in \Delta_K$ is a fixed (possibly unknown)
distribution over the actions. Regret against the Condorcet winner is
a special case, although a Condorcet winner may not exist for a given
preference matrix $\P$~\citep{Jamieson+15}. To connect this definition with our results,
note that our definition of best-response regret, Eqn.~\eqref{eq:br_reg}, upper bounds fixed-benchmark regret for any
$q^\star$:
\begin{fact}
	\label{prop:creg_vs_breg}
	For the non-contextual setting, we have $\mathrm{FB}\mbox{-}\mathrm{Regret}_T \le \mathrm{BR}\mbox{-}\mathrm{Regret}_T$, for any $q^\star \in \Delta_K$.
\end{fact}
As such, Theorem~\ref{thm:main} immediately yields guarantees for the
non-contextual fixed-benchmark setting. In particular, we obtain
$\tilde{O}(\sqrt{K^3T})$ worst-case fixed-benchmark regret, which is
slightly worse than the minimax optimal $\tilde{O}(\sqrt{KT})$
rate~\citep{CDB}. On the other hand, our regret notion is strictly stronger, and,
most importantly, our results generalize to the contextual setting
which does seem possible using techniques from this literature.

Beyond fixed-benchmark regret, notions involving
Borda~\citep{Busa14survey,Jamieson+15,falahatgar_nips} and Copeland
scores~\citep{Zoghi+15,komiyama+16,DTS} have also been considered. Our
work does not directly yield results for these notions. However, we
note that, as discussed by~\citet{CDB}, these notions fail the
\emph{independence of clones criterion}~\citep{schulze2011new}, which
makes them somewhat undesirable in contextual settings.  Finally we
note that many other variations of the non-contextual dueling bandits
problem have been considered, including adversarial preference
matrices~\citep{Adv_DB,ADB}, best-arm
identification~\citep{SGwin18,SG20,BTM}, full-ranking
\citep{Busa_pl,falahatgar_nips,SGrank18}, top-set detection
\citep{Busa_top,MohajerIcml+17,ChenSoda+18}, etc.
A very thorough literature survey on the recent developments in preference bandits can be found in \cite{Busa21survey,sui2018advancements}.  

\paragraph{Contextual dueling bandits.}
We are only aware of the following works that study the contextual setting: The
first is the work of~\citet{Yue+09}, which proposes a
policy-gradient style algorithm, and establishes a $T^{3/4}$ style
regret bound under convexity assumptions on the preference model. 
A follow-up work by \citet{kumagai2017regret} also uses gradient-based techniques to show an improved $T^{1/2}$ regret guarantee, but this result requires even stronger assumptions on the preference model. 
%% ; precisely it requires the underlying score-function to be twice continuously differentiable, ipschitz, strongly convex and smooth and a thrice differentiable, rotation-symmetric preference map. 

Another line of work considers contextual dueling bandits under a
special class of utility based preferences
\citep{pbo,Sui+17,S21}. \citet{S21} provides two $\tilde O(T^{1/2})$
regret algorithms assuming the preferences are a function of
underlying utility scores of the individual arms, where utility scores
of each arms are assumed to be a linear function of the
arm-features. \citet{pbo} makes a similar assumption and provides empirical results, but they do not establish any theoretical guarantees. 
%% although they
%% only have empirical results without any theoretical performance
%% guarantees. 
\citet{Sui+17} assumes the preference relations to be
generated from an unknown Gaussian process model, but also do not obtain regret bounds for their algorithms. 
%% however they did not
%% provide any regret analysis for their proposed algorithm either.
%
In comparison, we do not require any assumptions (beyond realizability)
on the preference model, and we still obtain a $T^{1/2}$ rate.

Closest to our work, is the paper of~\citet{CDB}, which studies the
contextual dueling bandits problem with an abstract policy set $\Pi:
\Xcal \to \Acal$ and without realizability. They propose a minimax
notion of regret given by:
\begin{align*}
	\mathrm{Policy}\mbox{-}\mathrm{Regret}_T := \max_{\pi \in \Pi} \sum_{t=1}^T \frac{1}{2} \sbr{f^\star(x_t)[\pi(x_t), a_t] + f^\star(x_t)[\pi(x_t),b_t]}.
\end{align*}
This definition differs from our best-response notion in two ways: (1)
the adversary chooses a policy for all $T$ rounds, rather than a
per-round distribution $\q_t$ and (2) we include the expectation over
the learner's randomness in our definition. Based on these
differences, we can show that our definition upper bounds the
definition of~\citet{CDB}:
\begin{restatable}[]{fact}{pvsbreg}
	\label{prop:preg_vs_breg}
	Let $\Pi$ be given, with $|\Pi| < \infty$. Then for any learner that chooses $(a_t,b_t) \sim \p_t$ at round $t$, we have
	\begin{align*}
		\mathrm{Policy}\mbox{-}\mathrm{Regret}_T \leq \breg_T + \sqrt{T \log (|\Pi|/\delta)},
	\end{align*}
	with probability $1-\delta$. 
\end{restatable}
Note that the minimax rate for policy regret is $\sqrt{KT \log
  |\Pi|}$, so the additive term is of lower order.  Additionally, we
are assuming realizability, without which it is impossible
to attain sublinear best-response regret, while minimizing policy
regret is always possible.

Apart from the difference in regret definition,~\citet{CDB} provides
two algorithmic results: an inefficient sparring algorithm that achieves
the optimal regret rate, and an $\epsilon$-greedy algorithm that
achieves a sub-optimal $T^{2/3}$-style regret assuming access to an
\emph{offline policy optimization oracle}. In comparison, we give
an algorithm with optimal regret assuming realizability as well as access
to an \emph{online square loss minimization oracle}. We emphasize that
the oracle models are quite different, so the results are not directly
comparable. However, previous experimentation with standard contextual
bandits suggests that algorithms based on square loss minimization may
be more effective in practice~\citep{dylan3}.

\paragraph{Markov games.}
Finally, we highlight a growing body of work on preference based
reinforcement learning in the Markov games
framework~\citep{littman1994markov}. Broadly speaking, a Markov game
models a multi-step decision making problem where several players
compete to maximize their payoff over the course of an episode. While
the specific formulations vary considerably, one formulation can be
seen as a multi-step generalization of (contextual) dueling bandits,
where regret is measured using our best response notion (which, as
discussed, subsumes many other notion in the
literature)~\citep{bai20,xie+20,bai+20}. In particular,~\citet{xie+20}
introduce the \emph{coarse correlated equilibrium} strategy and show
that it obtains $\sqrt{T}$ regret in the linear function approximation
setting, which directly gives an efficient contextual dueling bandits
algorithm under linear realizability. We highlight this technique in
Sec.~\ref{sec:linf} in an attempt to better connect these two lines of
work. On the other hand, our results for general function classes
under realizability are novel, and we hope that they find applications
in Markov games.

%%%%%%%%%%%%%%%%%%%%%%%%%%%%%%%

\section{Problem Setup}
\label{sec:prob}

%\subsection{Preliminaries}

\textbf{Notation.} Let $[n] := \{1,2, \ldots n\}$, for any $n \in \N$. Given a set $\cA$, for any two items $a,b \in \cA$, we use $a \succ b$ to denote that $a$ is preferred over $b$. 
%Lower case bold letters denote vectors, upper case bold letters denote matrices.
%For any function $f$, we abbreviate $f(\x)$ as $f_{x}$.
We use lower case bold letters for vectors and upper case bold letters for matrices.
$\I_d$ denotes the $d \times d$ identity matrix. For any vector $\x \in \R^d$, $\|\x\|_2$ denotes the $\ell_2$ norm of $\x$.
%and for any matrix $\bA \in \R^{m\times n}$, $\|\bA\|_2$ denotes the frobenius norm of matrix $\bA$. 
$\Delta_{K}:= \{\p \in [0,1]^K \mid \sum_{i = 1}^K p_i = 1, p_{i} \ge 0, \forall i \in [K]\}$ denotes the $K$-simplex, $\Delta_{K^2} = \Delta_{K \times K}$. 
$\e_i$ denotes the $i$-th standard basis vector in $\RR^K$. 
For this work, we consider the \emph{zero-sum} representation of preference matrices: 
$$
\mathcal P:= \{\P \in [-1,1]^{K \times K} \mid P[i,j] = -P[j,i], P[i,i]=0, ~~\forall i,j \in [K]\}\footnote{Standard dueling bandit literature generally represents preference matrices $\Q \in [0,1]^{K \times K}$, such that $Q[i,j]$ indicates the probability of item $i$ being preferred over item $j$. Here $\Q$ satisfies $Q[i,j] = 1-Q[j,i]$ and $Q[i,i] = 0.5$. Note both representations are equivalent as there exists a one to one mapping $\P = (2\Q -1) \in \cP$ \citep{CDB,Busa21survey}}.
$$
Note any $\P \in \mathcal P$ can be viewed as a \emph{zero-sum game}, where the two players, called row and column player resp., simultaneously choose two (possibly randomized) items from $[K]$, 
% or more generally two distributions over $n$ (or mixed strategies) w and u over rows and columns, 
with their goal being to respectively maximize and minimize the value of the selected entry.  
%\akcomment{I feel the first definition, in $[0,1]$ is not needed for our results and is just adding confusion.}

\iffalse%%%%%%%%%%% VN-Winner Defn %%%%%%%%%%%%%
\begin{defn}[Von-Neumann Winner of $\P$ \cite{dudik15}]
Given any (shifted) preference matrix $\P$, the Von-Neumann Winner of $\P$ is defined to be a probability distribution $\pi^\star \in \Delta_{n}$, such that
\[
\sum_{a=1}^{n} p(a)P[a,b] \ge 0, ~~\forall b \in [n]. 
\]
In words, for any item $b$, if $a$ is selected randomly according to $\p$, then the chance of beating $b$ in a duel (pairwise comparison) is at least $0.5$. This also implies that the same will be true even if $b$ is selected in a randomized way. 
\end{defn}

For any zero-sum matrix $\P \in D_{\zs}$, we denote the set of Von-Neumann's distributions of $\P$ by $\vn(\P):=\{\p \in \Delta_K \mid \E_{a \sim \p}[P[a,b]] \geq 0, ~\forall b \in [K]\}$

\begin{rem}
As discussed earlier, as $\P$ represents a zero-sum matrix game, by Von-Neumann’s celebrated minimax theorem:
\[
\max_{\p \in \Delta_n} \min_{\q \in \Delta_n} \p^\top \P \q = \min_{\q \in \Delta_n} \max_{\p \in \Delta_n} \p^\top \P \q,
\]
\end{rem}

\fi%%%%%%%%%%%%%%%%%%%%%%%%%%%%%%%

\subsection{Online Regression Oracle}
\label{sec:oreg}
An online regression oracle~\citep[][Chapter 3]{PLG06}, is an
algorithm, which we denote by \oreg, and which operates in the
following online protocol: on each round $t$ (1) it receives an
abstract input $z_t \in \cZ$, from some input space $\cZ$, chosen
adversarially by the environment, (2) it produces a real-valued
prediction $\hy_t \in \cY \subset \R$ where $\cY$ is some output space, and (3) it observes the true response $y_t
\in \cY$ and incurs loss $\ell(\hy_t,y_t) := (\hy_t - y_t)^2$.
%\akdelete{ which we denote by \oreg \, (a regression algorithm), takes as input an an abstract input $z_t \in \cZ$ chosen by the environment from some input space $\cZ$, and produces a real value $\hy_t \in \cY \subseteq \R$ from the output space $\cY$ at each  round $t$, upon which it gets to observe the true outcome $y_t \in \cY$.}
The goal of the oracle is to predict the outcomes as well as the best function in a given function class $\cF:= \{ f: \cZ \mapsto \R \}$, such that for every sequence of outcomes, the square loss regret is bounded.\footnote{The square loss itself does not play a crucial role, and can be replaced by other loss functions that is strongly convex with respect to the predictions~\citep{dylan1}.}

Formally, adopting the notation of~\cite{dylan1}, at time $t$ and for input $z \in \cZ$, \oreg\, can be seen as a mapping 
$\hy(z):=$\oreg$(z, \{z_\tau,y_\tau\}_{\tau = 1}^{t-1})$. Note this
%ŷt(x, a) ∶= SqAlgt(x, a ; (z1, y1), . . . , (zt−1, yt−1)), (5) 
corresponds to the prediction the algorithm would make at time $t$ if we passed in the input $z$, although this input may not be what is ultimately selected by the environment.
We will take $\cZ = (\cX \times [K]^2)$ to be the set of (context, action-pair) tuples such that $z_t = (x_t,a_t,b_t)$ with $x_t \in \cX$ and $(a_t,b_t) \in [K]^2$. Our output space is simply $\cY = [-1,1]$.
%The simplest condition under which our reduction works posits that SqAlg enjoys a regret bound for individual sequence prediction.

\begin{assum}
	\label{assump:oreg}
The online regression oracle \oreg\, guarantees for every sequence $\{z_t,y_t\}_{t \in [T]}$, its regret is bounded as
%\begin{align*}
$\sum_{t = 1}^T \big(\hy_t(z_t)-y_t\big)^2 - \inf_{f \in \cF}\sum_{t=1}^T \big( f(z_t) - y_t \big)^2 \le \regsq(T)$, where $\regsq(T) = o(T)$ is a known upper bound.
%\end{align*}
\end{assum}

%\akcomment{This is only true under realizability! (I edited   accordingly)} 
If we further assume realizability, in the sense that
there exists $f^\star \in \cF$ such that $\forall t: f^\star(z_t) =
\EE[y_t \mid z_t]$, then it is well-known that
Assumption~\ref{assump:oreg} further implies
\begin{align}
\sum_{t = 1}^T \big(\hy_t(z_t)- f^\star(z_t) \big)^2 \le \regsq(T).	\label{eq:oreg}
\end{align}
As we will see, under our realizability assumption on the preference
matrices, the underlying square
loss regression problem is also realizable, allowing us to appeal to Eqn.~\eqref{eq:oreg}.
%% nd so we will be able to
%% appeal to

%% Under the realizability assumption in Assumption~\ref{assum:fn_approx}
%% Above assumption is known to yield the following guarantee as analyzed in \cite{dylan1,dylan2}: Under Assumption \ref{assump:oreg}, for any sequence of $\{z_t\}_{t \in [T]}$, the online regression algorithm \oreg\, in turn satisfies:
%% \begin{align}
%% 	\label{eq:oreg}
%% \sum_{t = 1}^T \big(\hy_t(z_t)- f(z_t) \big)^2 \le \regsq(T)
%% \end{align}}

\begin{rem}[Some examples]
	\label{rem:cases}
Online square loss regression is a well-studied problem, and efficient algorithms with provable regret guarantees are known for many specific function classes including finite classes where $|\cF| < \infty$, finite and infinite dimensional linear classes, and others~\citep{dylan1,dylan2}. For completeness, we provide formal definition for some specific classes and instantiations of the regression oracles in Appendix~\ref{sec:eg_oreg}.
\end{rem}
	
\subsection{Setup and Objective}
%Let $A \in \N$ denotes the total number of actions of the row player. 
We assume a context set $\cX \subseteq \R^d$, action space of $K$ actions denoted by $\cA:= [K]$, and a function class $\cF = \{f: \cX \mapsto [-1,1]^{K \times K} \}$, all known to the learner ahead of the game. 
At each round, we assume a context-preference pair $(x_t,\P_t) \sim \cD$ is drawn from a joint-distribution $\cD$, such that $x_t \in \cX$, and $\P_t \in [-1,1]^{K \times K}$. 
The task of the learner is to select a pair of actions $(a_t,b_t) \in [K]\times [K]$, upon which an outcome $o_t \in \{\pm 1\}$ is revealed to the learner according to $\P_t$; specifically the probability that $a_t$ is preferred over $b_t$, indicated by $o_t = +1$, is given by $\Pr(o_t = 1):= \Pr(a_t \succ b_t) = \frac{P_t[a_t,b_t]+1}{2}$, and hence $\Pr(o_t = -1):= \Pr(b_t \succ a_t) = \frac{1-P_t[a_t,b_t]}{2}$.
%For simplicity, we would henceforth denote $o_t := o(a_t,b_t) \sim \P_t[a_t,b_t]$.
%For simplicity, we would henceforth denote the dueling set of round $t$ by $ab_t:= \{a_t,b_t\}$. %where $a_t \sim \p_t$, $b_t \sim \q_t$.

\begin{assum}[Realizability]
\label{assump:realizability}
Define $f^\star: x \mapsto \EE[\P \mid \x]$. We assume that $f^\star \in \cF$.
Thus, for any $a,b \in [K]$, we have $\EE[\P[a,b] \mid x] = f^\star(x)[a,b]$. 
\end{assum}

\iffalse%%%%%%%%%%%%%%%
Let $\Pi := \{\pi: \cX \mapsto \cA\}$ denotes the set of all policies. The objective of the learner is to minimize the expected regret w.r.t. the best policy of a given policy class $\Pi:= \{\pi \mid \cX \mapsto \cA\}$. More formally:

\begin{align}
\label{eq:reg}
\E[R_T^1] = \max_{\pi \in \Pi} \E\bigg[\sum_{t=1}^T \frac{P_t[\pi(\x_t),a_t] + P_t[\pi(\x_t),b_t]}{2}\bigg]
\end{align}
where the expectation is taken over the randomness of the environment $(x_t,\P_t) \sim \cD$, best (argmax of Eqn. \eqref{eq:reg}) policy $\pi$, and the randomness of the algorithm in playing the duel $\abt$. For any realization of learner's (randomized) actions $\{\abt \sim \p_t\}_{t \in [T]}$ (in general $\p_t \in \Delta_{[K]^2}$), we denote the best (argmax) policy in $\Pi$ by $\pi^\star$, and $\pi_t^\star:= \pi^\star(x_t) \in \Delta_n$. %i.e. $\pi^\star:= \arg\max_{\pi \in \Pi}$

\paragraph{Easier (Von-Neumann) notion of regret.}
Here is an easier notion of regret. Let $\pi^\star$ denote the von Neumann policy
\begin{align*}
    (\pi^\star_t,\pi^{'\star}_t) = \arg\min_{\pi} \max_{\pi'} \mathbb{E}\left[P_t[\pi(x), \pi'(x)]\right]
\end{align*}
Note $\pi^\star_t = \pi^{'\star}_t$ is also a von-Neumann \cite{dudik15} policy of the zero-sum matrix game $f^\star$
Then we measure the regret as
\begin{align}
\label{eq:vreg_cdb}
    \mathbb{E}[R_T^{vn}] = \mathbb{E}\left[\sum_{t=1}^T \frac{P_t[\pi_t^\star(x_t), a_t] + P_t[\pi_t^\star(x_t),b_t]}{2} \right]
\end{align}
\fi %%%%%%%%%%%%%%%%%%%%%%%%

\paragraph{Objective: Best-Response Regret} 
Assuming the learner selects the duel $(a_t,b_t) \sim \p_t \in \Delta_{K \times K}$ at each round $t$, we measure the learner's performance via a notion of best response regret, defined as:
\vspace{-8pt}
\begin{align*}
	\breg_T := \sum_{t=1}^T \max_{q_t \in \Delta_K} \frac{1}{2} \EE_{a \sim q_t} \EE_{(a_t,b_t) \sim \p_t} \sbr{f^\star(x_t)[a, a_t] + f^\star(x_t)[a,b_t]}.
\end{align*}

\iffalse%%%%%%%%%%%%%%%%%%%55
We also study a stronger regret with time variant realizability, when $f^\star$ can be chosen adversarially per round, say $f^\star_t \in \cF$ at round $t$. We refer this as \textrm{BR-Regret} w.r.t. the sequence $(f^\star_1,\ldots, f^\star_T)$):

\vspace{3pt} \noindent \textbf{\textrm{BR-Regret} for time-variant Realizabilty Sequence: An Even Stronger Regret}
\begin{align*}
	\wbreg_T := \sum_{t=1}^T \max_{q_t \in \Delta_K} \frac{1}{2} \EE_{a \sim q_t} \EE_{(a_t,b_t) \sim \p_t} \sbr{f^\star_t(x_t)[a, a_t] + f^\star_t(x_t)[a,b_t]}
\end{align*}
\fi%%%%%%%%%%%%%%%%%%%%%%%%555

\begin{rem}[Learner's Obligation to Randomize]
	\label{rem:imposs}
    As stated, we allow the learner to choose its actions randomly and
    our regret definition includes an expectation over this
    randomness. An alternative would be to measure the best response
    regret on the realized outcomes $(a_t,b_t)$ chosen by the learner
    at each time:
      \begin{align}
      \label{eq:br_reg_bad}
        \sum_{t=1}^T \max_{\q_t \in \Delta_K} \frac{1}{2} \EE_{a \sim q_t}\sbr{ f^\star(x_t)[a,a_t] + f^\star(x_t)[a,b_t]}.
      \end{align}
      The two regret definitions ($\breg_T$ vs the one defined in \eqref{eq:br_reg_bad}) are of course equivalent if the learner does
      not randomize, but they are very different in general. In fact,
      if we measure regret on the realized outcomes, as displayed
      above (irrespective of whether the learner is allowed to randomize or not), then sublinear regret is not possible in general, since
      the preference matrices $f^\star(\cdot)$ may not have
      pure-strategy Nash equilibria. See Appendix \ref{app:imposs} for a concrete example.
      %\akedit{(See appendix for a concrete example)}. 
      On the other hand, if the learner
      randomizes and we incorporate this into the regret definition,
      then in principle the learner could set both marginals of $\p_t$
      to be a Nash equilibrium for $f^\star(x_t)$ to guarantee 0
      regret. 
 %   \akcomment{edited significantly, please move example to the appendix.}
\end{rem}

\section{Warm up: Structured Function Classes}
\label{sec:linf}

As a warm up, and to highlight an overlooked connection between
dueling bandits and Markov games, we briefly sketch how UCB-based
algorithms can achieve $\sqrt{T}$-regret in some structured dueling
bandits settings. These algorithms have appeared previously in the
Markov games literature~\citep{bai+20,xie+20}, so we summarize the key
ideas here and defer additional details to the appendices. While these
ideas are not technically novel, in light of the relationship between
best-response regret and previously studied notions in the dueling
bandit literature, we believe it is worthwhile to bring these
techniques to the attention of the dueling bandit community.

We begin with the standard ``non-contextual'' dueling bandits setting
where there is just a single unknown preference matrix $\P \in
[-1,1]^{K\times K}$. Here, it is natural to deploy a confidence-based
strategy that, at round $t$, maintains an estimate
$\boldsymbol{\hP}_t$ of the underlying parameter and a confidence set
$C_t[i,j] := \tilde{O}\bigg(\sqrt{\frac{1}{N_{t}[i,j]}}\bigg)$ for
each entry, where $N_t[i,j]$ is the number of times that entry $(i,j)$
has been dueled prior to round $t$. The critical component of the
algorithm, and the main departure from standard UCB approaches, is the
action selection scheme. Here we find a \emph{coarse correlated
equilibrium} (CCE) of the ``upper confidence'' matrix
$\boldsymbol{U}_t := \boldsymbol{\hP}_t + \C_t$, defined as any joint
distribution $\p_t \in \Delta_{K\times K}$ that satisfies:
\begin{align*}
 &  \sum_{\mathclap{a,b \in [K]\times[K]}} p_t[a,b] {U}_t[a,b] \geq \max_{a^\star \in [K]} \sum_{b \in [K]\times[K]}p_t^r[b]U_t[a^\star,b],\\
  & \sum_{\mathclap{a,b \in [K]\times[K]}} p_t[a,b] {U}_t[b,a] \geq \max_{b^\star \in [K]} \sum_{a \in [K]\times[K]}p_t^\ell[a]U_t[b^\star,a],
\end{align*}
where $p_t^\ell[\cdot] = \sum_b p_t[\cdot,b]$ is the ``left''
marginal and $p_t^r$ is the analogously defined ``right'' marginal.
In words, the CCE is a joint distribution over the actions of the two
players, such that neither player is incentivized to deviate
unilaterally from their marginal strategy~\citep{pdbook}. Since the
matrix $\U_t$ is not zero-sum, a Nash equilibrium may not be
efficiently computable \citep{pdbook,daskalakis2009complexity}; however, a CCE is
guaranteed to exist and is easily computed by solving the above linear
feasiblity problem. Returning to the algorithm, we
find a CCE solution $\p_t$ for the upper confidence matrix $\U_t$,
sample $(a_t,b_t) \sim \p_t$, observe the outcome, and update
our statistics for the next round. We refer this algorithm as \algc. The full pseudocode is presented in Appendix~\ref{app:stdb_pseudocode} (see Algorithm \ref{alg:stdb}).

\subsection{Regret Analysis for CCE based Algorithms}
The algorithm summarized above achieves the following regret
guarantee. This theorem essentially appears in both~\citet{xie+20}
and~\citet{bai+20} in more general forms; both study multi-step Markov
games,~\citet{xie+20} considers linear function approximation,
and~\cite{bai+20} allows the two players to have different action set
sizes.
%% ~\citet{xie+20} studies for a related result in markov games which boils down to $K$-armed dueling bandit setup for $H = 1, S = 1, A = B = K$. Thm. $1$ of \cite{xie+20} also yields a similar regret bound if we set $d = K$ and $H = S = 1$ in their setting modulo a slightly suboptimal $\tilde O(\sqrt K)$ multiplicative factor.

\begin{thm}[Regret of \algc~(Alg. \ref{alg:stdb}), informal]
\label{thm:stdb_inf}
  In the non-contextual standard $K$-armed dueling bandits setting, the above algorithm  has regret
  \begin{align*}
    \breg_T \leq O(K \log (KT)\sqrt{T}).
  \end{align*}
\end{thm}

%\akcomment{add theorem ref, citation, and word informal to the theorem ``title''}
An intuitive proof sketch of the algorithm is given below (while the complete analysis is presented in Appendix~\ref{app:stdb_reg}). 
We mention a few remarks, before describing the key step in the analysis.
\begin{enumerate}[parsep=0pt]
\item The above bound is optimal in the dependence on $T$, up to
  logarithmic factors, as the worst-case lower bound is known to be
  $\Omega(\sqrt{T})$ even when assuming existence of a Condorcet
  Winner~\citep{Komiyama+15,CDB}. Recall from Fact~\ref{prop:creg_vs_breg}, that the best-response
  regret upper bounds the regret to any fixed benchmark, including a
  Condorcet winner (if it exists).
\item While near-optimal in its dependence on $T$, the dependence on
  the number of arms is sub-optimal, as it is possible to achieve
  $\tilde{O}(\sqrt{KT})$ best-response regret in the non-contextual
  setting. Indeed, this optimal rate can be achieved here and in other related settings by sparring
  optimal adversarial bandit algorithms, such as Exp3~\citep{CDB,Ailon+14,Adv_DB,Sui+17}
%% applies sparring-EXP3 to derive a $\breg_T$ of $\tO(\sqrt{K T})$ for standard $K$-armed dueling bandits, while \citet{Ailon+14} and \cite{Adv_DB}
%%   obtain a similar rate but for the simpler Fixed-Benchmark (FB) regret (see \textrm{FB-Regret}$_T$, Eqn. \eqref{eq:stdb_reg}) and for a restricted class of only utility based preference matrices \citep{Sui+17}. 
  Unfortunately, all algorithms that achieve
  the optimal rate rely heavily on adversarial online learning
  techniques and do not seem to yield efficient algorithms in the more
  general contextual setting, and so we believe it is worthwhile to
  also study algorithms with a more statistical flavor.
\item The result above can be generalized to any setting where valid
  and shrinking confidence intervals can be constructed, including
  linear and generalized linear dueling bandit settings under
  realizability. These results are presented in Appendix \ref{sec:linf_o}, where we
  give an $\tilde{O}(d\sqrt{T})$ regret algorithm (Alg.~\ref{alg:linf}) for the linear case (see Theorem~\ref{thm:linf}, Appendix \ref{sec:linf_o}).
\end{enumerate}

\paragraph{Proof Sketch: Regret Analysis of \algc ~(Alg. \ref{alg:stdb}).} Turning to the analysis, the first part of the analysis involves
verifying the validity of the confidence intervals, which is quite
standard. The more interesting part involves relating the regret of
$\p_t$ to the confidence bounds, where we must crucially use the fact
that $\p_t$ is the CCE for the upper confidence matrix $\U_t$. The
essential calculation is as follows: for any $\q \in
\Delta_K$, we have
\begin{align*}
  \q^\top \P \p_t^\ell\  {\buildrel (i)\over\leq}\  \q^\top \U_t\p_t^\ell\  {\buildrel (ii) \over \leq}\  \max_{b^\star}\sum_{a \in [K]} p_t^\ell(a)U_t[b^\star,a]\  {\buildrel{(iii)} \over \leq}\  \sum_{a,b \in [K]\times[K]}p_t[a,b]U_t[b,a],
\end{align*}
where (i) follows from the upper confidence property of $\U_t$, (ii)
follows since $\q \in \Delta_K$ and (iii) follows from the fact that
$\p_t$ is a CCE for $\U_t$. 

The exact same calculation applies for the
right player $p_t^r$ and yields the bound $\sum_{a,b}
p_t(a,b)U_t[a,b]$. Finally using the fact that
$U_t[a,b] = -U_t[b,a] + 2C_t[a,b]$ we obtain
\begin{align*}
\q^\top \P (\p_t^\ell + \p_t^r) \leq \sum_{a,b} p_t[a,b](U_t[a,b]+ U_t[b,a])= \sum_{a,b} p_t[a,b] C_t[a,b].
\end{align*}
In other words, we can bound the per-round regret by the confidence
width of the actions chosen by the learner. This enables us to use
standard potential-function arguments for bounding the confidence sum,
which yields the final regret bound.

\section{Main result: General function classes}
\label{sec:genf}

While the above CCE-based algorithm yields $\sqrt{T}$-type regret for
function classes that admit pointwise confidence intervals, this is
only possible for certain structured function classes. In this
section, we turn to our main result: an efficient algorithm for more
general classes $\cF$. 

Our algorithm is an adaptation of the contextual bandit algorithm
\texttt{SquareCB}, due to~\citet{dylan1}, which uses an online square
loss oracle to make predictions and an inverse gap-weighting scheme
that uses these predictions for action selection.  Their key lemma
establishes a \emph{per-round inequality} that relates the contextual
bandit regret to the square loss of the online oracle. This inequality
is established in a minimax sense, where no assumptions are made about
the predictions of the oracle or the true reward function.

We follow their recipe and instantiate the square loss oracle with
instance space $\Zcal := \Xcal \times [K] \times [K]$. Then at each
round $t \in [T]$, after observing the context $x_t$, we can query the
oracle for predictions on $(x_t,a,b)$ for each $(a,b) \in
[K]\times[K]$ and collect the predictions into a zero-sum matrix
$\widehat{\Y}_t \in [-1,1]^{K \times K}$. Now, the goal is to use
$\widehat{\Y}_t$ to construct a distribution $\p_t \in \Delta_{K \times
  K}$ such that
\begin{align*}
\max_{f^\star \in D_{\zs}} \underbrace{\max_{\q \in \Delta_K} \frac{1}{2} \EE_{a^\star \sim \q, (a_t,b_t) \sim \p_t}\sbr{ f^\star[a^\star,a_t] + f^\star[a^\star,b_t]}}_{\textrm{best-response regret of $\p_t$}} - \frac{\gamma}{4} \underbrace{\EE_{(a_t,b_t) \sim \p_t}\sbr{ (\widehat{Y}_t[a_t,b_t] - f^\star[a_t,b_t])^2}}_{\textrm{square loss of $\widehat{\Y}_t$}} \leq \frac{\textrm{poly}(K)}{\gamma}.
\end{align*}
Here $D_{\zs} := \{ \M \in [-1,1]^{K\times K} \mid M[i,j] = -M[i,j],
M[i,i] = 0\}$ is the set of zero-sum matrices and $\gamma>0$ is a
exploration rate parameter. In words, we are asking that, no matter
the true preference matrix $f^\star$, $\p_t$ has best response regret
that is upper bounded by the square loss of $\widehat{\Y}_t$ on actions
chosen from $\p_t$, up to an additive $\textrm{poly}(K)/\gamma$
term. If we can find such a $\p_t$, then we can sample $(a_t,b_t)\sim
\p_t$, observe the outcome $o_t$ and pass the example $z_t =
(x_t,a_t,b_t)$ along with outcome $o_t$ to the square loss
oracle. This ensures that the observed loss is an unbiased estimate
for the second term above, so if we add up the per-round inequality
for all $t \in [T]$, we obtain
\begin{align*}
\breg_T &\leq \frac{\gamma}{4} \cdot \sum_{t=1}^T \EE_{(a_t,b_t) \sim \p_t} \sbr{ (\widehat{Y}_t[a_t,b_t] - f^\star(x_t)[a_t,b_t])^2 } + \frac{\textrm{poly}(K)\cdot T}{\gamma} \\
& \leq \frac{\gamma}{4}\cdot \regsq(T) + \frac{\textrm{poly}(K)\cdot T}{\gamma} = O\rbr{\sqrt{\textrm{poly}(K)\cdot T\cdot\regsq(T) }}.
\end{align*}
Here, the second inequality follows from Assumption~\ref{assump:oreg},
in particular, Eqn.~\eqref{eq:oreg} and the fact that the example $z_t
= (x_t,a_t,b_t)$ that we feed to the square loss oracle will have
$(a_t,b_t) \sim \p_t$. Then the final bound is based on tuning $\gamma
\asymp \sqrt{\frac{\textrm{poly}(K) T}{\regsq(T)}}$.

Thus, the main remaining step is to establish the per-round regret
inequality. This is where our analysis departs from that
of~\citet{dylan1}, as we must account for the game-theoretic structure
of the best response regret definition. To proceed with this analysis, let
us define the per-round minimax value as:
\begin{align}
  V(\gamma) := \max_{\Y \in D_{\zs}} \min_{\p \in \Delta_{K\times K}} \max_{\substack{f^\star \in D_{\zs} \\ \q \in \Delta_K}} \sbr{\sum_{(a,b)}q[a] f^\star[a,b]\frac{p^\ell[b] + p^{r}[b]}{2} - \frac{\gamma}{4}\sum_{i \ne j} p[i,j](f^\star[i,j] - Y[i,j]^2) },
\end{align}
where $p^{\ell}[\cdot] = \sum_{b} p[\cdot,b]$ is the ``left''
marginal and $p^r[\cdot] = \sum_a p[a,\cdot]$ is the ``right''
marginal of $\p$. The following lemma shows that the minimax value
$V(\gamma)$ is bounded by $O(K/\gamma)$.

\begin{lemma}
	\label{lem:minmax}
For any $\gamma \geq 2K$, $V(\gamma) \le \frac{5K}{\gamma}$.
\end{lemma}

\begin{proof}
  By a change of variables, we can write
  \begin{align*}
    V(\gamma) = \max_{\Y \in D_{\zs}} \min_{\p \in \Delta_{K\times K}} \max_{\substack{\delta \in D_{\zs} \\ \q \in \Delta_K}} \sbr{\sum_{(a,b)}q[a] (Y[a,b]+\delta[a,b])\frac{p^\ell[b] + p^{r}[b]}{2} - \frac{\gamma}{4}\sum_{i \ne j} p[i,j]\delta[i,j]^2 }.
  \end{align*}
  Next, we relax all constraints on $\delta_{a,b}$, fix $\p$ and $\q$, and
  optimize over $\delta_{a,b}$. Maximizing the quadratic for each
  $\delta_{a,b}$, we find that setting $\delta_{a,b} =
  \frac{q[a](p^\ell[b] + p^r[b])}{\gamma p[a,b]}$ yields an upper
  bound on $V(\gamma)$:
  \begin{align}
    V(\gamma) &\leq \max_{\Y \in D_{\zs}}\min_{\p \in \Delta_{K\times K}} \max_{\q \in \Delta_K}\sbr{ \sum_{a,b} q[a]Y[a,b]\frac{p^\ell[b]+p^r[b]}{2} + \frac{1}{2\gamma} \sum_{a,b}\frac{q[a]^2 (p^\ell[b]+p^r[b])^2}{p[a,b]}} \notag\\
    & \leq \max_{\Y \in D_{\zs}} \min_{\p \in \Delta_{K}} \max_{\q \in \Delta_K}\sbr{ \sum_{a,b} q[a]Y[a,b]p[b] + \frac{2}{\gamma} \sum_{a,b}\frac{q[a]^2 p[b]^2}{p[a]p[b]}} \notag\\
    & \leq \max_{\Y \in D_{\zs}} \min_{\p \in \Delta_{K}} \max_{\q \in \Delta_K}\sbr{ \sum_{a,b} q[a]Y[a,b]p[b] + \frac{2}{\gamma} \sum_{a}\frac{q[a]}{p[a]}}. \label{eq:pt_def}
  \end{align}
  Above, the first inequality uses the maximizing value of
  $\delta_{a,b}$, while in the second inequality, we restrict the
  minimizing player $\p$ to sample $(a,b)$ iid from a marginal
  distribution (which we overload and also call $\p$). In the last
  inequality we use the fact that $\p,\q\in\Delta_K$ so that, e.g.,
  $q[a]^2 \leq q[a]$.

  We bound the final term by fixing $\Y$ and applying the minimax
  theorem. To do so, observe that the objective is linear (and hence
  concave) in $\q$ and convex in $\p$. To ensure that the objective is
  defined everywhere, we further shrink the domain for $\p$ by
  \emph{smoothing}: Fixing $\epsilon > 0$, for any $\p \in \Delta_{K}$
  we define $\p^{(\epsilon)} := (1-\epsilon) \p + \epsilon \one/K$. As
  $\p^{(\epsilon)}$ itself is a distribution, this upper bounds our
  objective while ensuring that the conditions for applying the
  minimax swap are satisfied. As such, we obtain
  \begin{align*}
    V(\gamma) & \leq \max_{\Y \in D_{\zs}} \min_{\p \in \Delta_{K}} \max_{\q \in \Delta_K}\sbr{ \sum_{a,b} q[a]Y[a,b]p^{(\epsilon)}[b] + \frac{2}{\gamma} \sum_{a}\frac{q[a]}{p^{(\epsilon)}[a]}}\\
    & = \max_{\Y \in D_{\zs}} \max_{\q \in \Delta_K} \min_{\p \in \Delta_{K}} \sbr{ \sum_{a,b} q[a]Y[a,b]p^{(\epsilon)}[b] + \frac{2}{\gamma} \sum_{a}\frac{q[a]}{p^{(\epsilon)}[a]}}\\
    & \leq \max_{\Y \in D_{\zs}} \max_{\q \in \Delta_K} \sbr{ \sum_{a,b} q[a]Y[a,b]q^{(\epsilon)}[b] + \frac{2}{\gamma} \sum_{a}\frac{q[a]}{q^{(\epsilon)}[a]}}\\
    & = \max_{\Y \in D_{\zs}} \max_{\q \in \Delta_K} \sbr{ (1-\epsilon) \sum_{a,b} q[a]Y[a,b]q[b] + \frac{\epsilon}{K} \sum_{a,b}q[a]Y[a,b] + \frac{2}{\gamma} \sum_{a}\frac{q[a]}{(1-\epsilon)q[a] + \epsilon/K}}.
  \end{align*}
  Here, the first inequality restricts the domain for $\p$ using the
  smoothing operator, the first equality is the minimax swap, and the
  second inequality follows by choosing $\p = \q$. The remaining three
  terms are bounded as follows: (i) the first term is zero since $\Y
  \in D_{\zs}$, (ii) the second term is trivially upper bounded by
  $\epsilon$, (iii) the third term is at most $4K/\gamma$ as long as
  $\epsilon \leq 1/2$. Setting $\epsilon = K/\gamma$, we obtain the
  result.
\end{proof}

Lemma~\ref{lem:minmax}, combined with the above discussion,
immediately certifies the existence of a strategy that achieves
$O(\sqrt{KT\cdot \regsq(T)})$ regret for contextual dueling bandits
with realizability. For constructing an algorithm, the missing piece
is the action selection scheme $\p_t$, whose existence is guaranteed
by Lemma~\ref{lem:minmax}. For this, an examination of
Eqn.~\eqref{eq:pt_def} reveals that we can compute a suitable $\p_t$
by solving a simple convex program in the action space. Specifically,
given predictions $\widehat{\Y}_t$ on round $t$, we define $\p_t \in \Delta_K$ as any solution to the following convex feasibility problem:
\begin{align}
	\label{eq:Vgam}
  \p_t \textrm{ satisfies } \forall i \in [K]: \sum_b \widehat{Y}_t[i,b]p_t[b] + \frac{2}{\gamma} \frac{1}{p_t[i]} \leq \frac{5K}{\gamma}.
\end{align}
The proof of Lemma~\ref{lem:minmax} shows that this program is always
feasible\footnote{Note that we can add
an $O(K/\gamma)$ slack term in the RHS of Eqn. \eqref{eq:Vgam} to the tolerate the approximations arising from numerical methods. This affects the final result only in constant factors.} and that $\p_t$
provides the per-round inequality that we require.
%\akcomment{Please move shorter calculation to appendix.}
%\akcomment{I think no need for algorithm names, we can just use Algorithm 1, Algorithm 2}

\begin{algorithm}[h]
  \caption{\textbf{\algf}}%
  \label{alg:genf}
  \begin{algorithmic}[1]	
	\STATE {\bfseries input:} Arm set: $[K]$, parameters $\gamma > 0$. 
	\STATE ~~~~~~~~~~ An instance of \oreg\, for function class $\cF$
	\FOR{$t = 1, 2, \ldots, T$}
	\STATE Receive context $x_t$
	\STATE $\forall$ $a < b, \, a,b \in [K]\times [K]$: Query $\hy(x_t,a,b) \gets \oreg\Big(\{(x_\tau,a_\tau,b_\tau),y_\tau\}_{\tau = 1}^{t-1}\Big)$ 
    \STATE Collect predictions into matrix $\widehat{\Y}_t$ and find $\p_t\in\Delta_K$ satisfying %\ascomment{change to $1/p_t[i]$}
    \begin{align*}
      \forall i \in [K]: \sum_{b \in [K]} \widehat{Y}_t[i,b]p_t[b] + \frac{2}{\gamma}\frac{1}{p_t[i]} \leq \frac{5K}{\gamma}.
    \end{align*}%\ascomment{why not refer to Eqn 8 directly?}
    \STATE Sample $a_t,b_t \iidsim \p_t$, play the duel $(a_t,b_t)$ and receive feedback $o_t$.
    \STATE Update $\oreg$ with example $(x_t,a_t,b_t)$ and label $o_t$. 
	\ENDFOR 
  \end{algorithmic}
\end{algorithm}

We put all the pieces together to obtain our final algorithm, with
pseudocode displayed in Algorithm~\ref{alg:genf}. The main guarantee
for the algorithm is as follows:
\begin{thm}
  \label{thm:minmax}
  Under Assumptions~\ref{assump:oreg} and~\ref{assump:realizability},
  Algorithm~\ref{alg:genf} with $\gamma = \sqrt{\frac{20KT}{\regsq(T)}}$ ensures that
  $\breg_T \leq \sqrt{5KT\regsq(T)}$ for any $T \geq
  4K\regsq(T)$. 
\end{thm}
The restriction on $T$ arises since
Lemma~\ref{lem:minmax} applies only when $\gamma \geq 2K$. The complete proof of Theorem~\ref{thm:minmax} is given in Appendix~\ref{app:main_thm_proof}. Moreover, we further note that given
the choice of $\p_t$ in Eqn. \ref{eq:Vgam}, there is actually a simpler argument for our reduction and derivation of Theorem \ref{thm:minmax}. We provide this analysis in
Appendix~\ref{app:simpler_calculation}.

By instantiating the square loss oracle appropriately, we obtain
end-to-end guarantees for many function classes of interest. Some of
these results are summarized in the next corollary.

\begin{cor}
  \label{cor:minmax} Algorithm~\ref{alg:genf} yields the following best-response regret guarantees:
  \begin{itemize}
    \item For $\cF$ with $|\cF| < \infty$, instantiating $\oreg$ as the exponential
      weights algorithm guarantees $\regsq(T) \leq O(\log |\Fcal|)$
      and hence $\breg(T) \leq O(\sqrt{KT\log |\Fcal|})$. For example in the  $K$-armed (non-contextual) dueling bandit setting, one can 
      construct $\cF$ such that $\log |\cF| = \tilde{O}\big(K^2\big)$ to obtain $\breg(T)\leq \tilde{O}\big(\sqrt{K^3T}\big)$. 
    \item For low dimensional linear predictors $\cF = \{ (x,a,b) \mapsto \inner{\theta}{\phi(x,a,b)} : \theta \in \RR^d, \nbr{\theta}_2 \leq 1\}$, instantiating $\oreg$ as the Vovk-Azoury-Warmuth forecaster guarantees $\regsq(T) \leq O(d \log(T/d))$ and hence $\breg(T) \leq O(\sqrt{dKT \log(T/d)})$.
    \item Alternatively, for linear predictors, instantiating $\oreg$ as online gradient descent guarantees $\regsq(T) \leq O(\sqrt{T})$ and hence $\breg(T) \leq O(K^{1/2} T^{3/4})$. 
  \end{itemize}
\end{cor}
Please see Appendix~\ref{app:spcl_reg} for details and additional
examples. Note that the non-contextual rate of $\tilde{O}\big(\sqrt{K^3T}\big)$ is sub-optimal by an $O(K)$ factor. We close this section with two final remarks.

\begin{rem}
  Formally, Theorem~\ref{thm:minmax} only requires
  Eqn.~\eqref{eq:oreg} to hold, and this may be possible in somewhat
  more general settings than the realizability condition stated in
  Assumption~\ref{assump:realizability}. One example is the
  time-varying or dynamic regret setting, where preferences at round
  $t$ are governed by $f_t^\star \in \cF$ and we assume the sequence
  $(f_1^\star,\ldots,f_T^\star)$ has small total variation or path
  length. In such cases, one can achieve Eqn.~\eqref{eq:oreg} with
  non-trivial $\regsq(T)$~\citep{raj2020non,baby2021optimal}, which
  can then be used in Theorem~\ref{thm:minmax}.
\end{rem}

\begin{rem}[Computational Complexity]
The two main computational bottlenecks in Algorithm~\ref{alg:genf} are
the square loss oracle itself and the computation of $\p_t$ in each
iteration. The former is efficient for many function classes of
interest, while the latter involves (approximately) solving a convex
feasibility problem in $K$ dimensions with $O(K)$ constraints, which
can be done in $\textrm{poly}(K)$ time. Thus the algorithm incurs a
$\textrm{poly}(K)$ computational overhead over the square loss oracle.
\end{rem}

\section{Discussion: A Barrier for Oracle-Efficient Agnostic Algorithms}
\label{sec:example}

As we have seen, realizability of the payoff matrices permits
computationally efficient algorithms for contextual dueling bandits
with optimal $\sqrt{T}$-type regret. At the same time, the classical
approach of sparring Exp4 achieves a similar regret guarantee in the
\emph{agnostic} setting (e.g., without realizability), but
unfortunately it is not computationally efficient for most policy
classes of interest~\citep{CDB}.  It is thus natural
to ask if we can design computationally tractable algorithms for
contextual dueling bandits in the absence of realizability.

In the standard contextual bandit setting, computational tractability
for the agnostic setting is formalized by providing the algorithm a
policy class $\Pi: \Xcal \to \Acal$ and an \emph{optimization oracle}
over $\Pi$. The optimization oracle serves as an abstraction of
supervised learning, allowing the algorithm to efficiently search over
the class $\Pi$, and it leads to algorithms that can be implemented
via a reduction to supervised learning as a
primitive~\citep{dudik2011efficient,agarwal14,krishnamurthy2015contextual,rakhlin2016bistro,syrgkanis2016efficient,syrgkanis2016improved,luo17}. Such
algorithms are called \emph{oracle-efficient}.

In this section, we provide some evidence to suggest that
significantly new techniques are required to develop oracle-efficient
algorithms for agnostic contextual dueling bandits with $\sqrt{T}$
regret. The main observation is that all such algorithms for
standard contextual bandits establish some concentration inequality on
the regret of all policies $\pi \in \Pi$, but establishing such a
guarantee in the dueling setting requires incurring large
regret. Indeed, Lem. 13 in~\citet{agarwal14} asserts that
\begin{align}
	\forall \pi \in \Pi: \Reg(\pi) \lesssim 2\widehat{\Reg}_t(\pi) + c_0\sqrt{\frac{1}{t}}, \qquad \widehat{\Reg}_t(\pi) \leq 2 \Reg(\pi) + c_0\sqrt{\frac{1}{t}}, \label{eq:cb_conc}
\end{align}
where $\Reg(\pi)$ is the population regret for $\pi$, and
$\widehat{\Reg}_t(\pi)$ is an importance weighted empirical estimate
based on $t$ rounds of interaction, and $c_0>0$ is some constant that
captures other problem parameters (e.g., number of actions, size of
policy class, etc.), but depends at most logarithmically on $t$. This
guarantee is central to the regret analysis, and similar bounds appear
in related works.

In the dueling setting, we define $\Reg(\pi) = \max_{\pi'}
\EE_{(x,\P)}\sbr{ P[\pi'(x),\pi(x)]}$ and extend this to distributions
over policies in the obvious way. Then, as a step toward porting the
proof technique from~\citet{agarwal14} to the dueling setting,
we can ask if there is an estimator $\widehat{\Reg}_t(\cdot)$ that
achieves the guarantee in~\pref{eq:cb_conc} for this
definition. Unfortunately, this is not possible without the algorithm
incurring $\Omega(T)$ regret.

\begin{proposition}
	\label{prop:hardness}
	Consider any algorithm $\textsc{Alg}$ that produces estimates
	$\widehat{\Reg}_t(\cdot)$ that satisfy~\pref{eq:cb_conc} for all
	$t \leq T$ and all policies $\pi$ in some given class $\Pi$. Then
	there is a contextual dueling bandit instance where $\textsc{Alg}$
	incurs $\Omega(T/c_0^2)$ regret.
\end{proposition}
\begin{proof}[Sketch]
	Consider the non-contextual $3 \times 3$ preference matrix instance parametrized by $\epsilon$:% with $3$ actions and expected payoff matrix:
	\begin{align*}
		\EE\sbr{P} = \rbr{\begin{matrix}
				0 & 1 & 0\\
				-1 & 0 & \epsilon\\
				0 & -\epsilon & 0
		\end{matrix}}.
	\end{align*}
	Let us label the actions $a$, $b$, and $c$ and define $\Pi=\{a,b,c\}$. Here action $a$ is a
	Condorcet winner, while 
	%% it is always the argmax policy in the
	%%   definitions of regret for both the algorithm and any single
	%%   action. 
	action $c$ is near optimal, with $\Reg(c) = \epsilon$. However, to
	estimate the regret of action $c$, we require playing the pair
	$(b,c)$. This poses a problem, since we can use action $a$ as the
	comparator when calculating the dueling regret of $\textsc{Alg}$,
	which shows that \textsc{Alg}'s regret is \emph{lower bounded} by
	the number of times that it plays action $b$.
	
	In more detail, at time $T$, the two guarantees in~\pref{eq:cb_conc}
	imply that $\widehat{\Reg}_T(c) \in [\frac{1}{2}(\epsilon -
	c_0/\sqrt{T}), 2\epsilon + c_0/\sqrt{T}]$. Now, consider two
	instances, one where $\epsilon = \epsilon_1 := c_0/\sqrt{T}$ and the
	other, where $\epsilon = \epsilon_2 := 8c_0/\sqrt{T}$. The intervals
	for $\widehat{\Reg}_T(c)$ do not intersect, and so, if \textsc{Alg}
	guarantees~\pref{eq:cb_conc} we can use the value of
	$\widehat{\Reg}_T(c)$ as a test statistic to distinguish between
	these two instances.
	On the other hand, testing between these two instances is equivalent
	to testing whether the mean of a Bernoulli random variable is
	$(1+\epsilon_1)/2$ or $(1+\epsilon_2)/2$ from iid samples. The
	number of samples available for this problem is the number
	of times $\textsc{Alg}$ plays the pair $(b,c)$. A standard lower
	bound argument reveals that, for $T$ large enough, $\textsc{Alg}$
	must play action $b$ at least $\Omega(T/c_0^2)$ times, which proves
	the claim.
\end{proof}

We emphasize that the above claim only shows that establishing a
certain intermediate guarantee is not possible in dueling contextual
bandits. It is a somewhat weak form of hardness that does not rule out
oracle-efficient agnostic algorithms. On the other hand, as all such
algorithms for standard contextual bandits do make claims similar
to~\pref{eq:cb_conc},~\pref{prop:hardness} suggests that fundamentally
new techniques are required to obtain agnostic algorithms for this setting. We believe that this is quite an interesting open
problem.

%\section{Conclusion}

%\red{todo}

\subsection*{Acknowledgements}
AK thanks Akshay Balsubramani, Alekh Agarwal, Miroslav Dud\'{i}k, and
Robert E. Schapire for fruitful discussions regarding the result in
Section~\ref{sec:example}.

%\newpage

\bibliographystyle{plainnat}
\bibliography{contextual_dueling}

\appendix

\onecolumn
{
\section*{\centering\LARGE{Supplementary: \papertitle}}
\vspace*{1cm}
\allowdisplaybreaks

\section{Appendix for Sec. \ref{sec:rel}}

\subsection{Proof of Fact. \ref{prop:preg_vs_breg}}
\label{app:pvsbreg}

\pvsbreg*

\begin{proof}
	Since $(a_t,b_t) \sim \p_t$, we can use Hoeffding's inequality and a union bound to deduce that $\forall \pi \in \Pi:$
	\begin{align*}
		&\sum_{t=1}^T \frac{1}{2} \sbr{f^\star(x_t)[\pi(x_t),a_t] + f^\star(x_t)[\pi(x_t),b_t]} \\
		& \hspace{1in} \leq \sum_{t=1}^T \frac{1}{2} \EE_{\p_t} \sbr{f^\star(x_t)[\pi(x_t),a_t] + f^\star(x_t)[\pi(x_t),b_t]} + \Ocal\rbr{\sqrt{T \log |\Pi|/\delta}},
	\end{align*}
	with probability $1-\delta$. Now we can easily translate from \textrm{Policy-Regret} to $\breg$\, by pushing the $\max_{\pi \in \Pi}$ inside the summation only yields an upper bound, justifying the claim.
\end{proof}

\section{Appendix for Sec. \ref{sec:prob}}

\subsection{Examples: Some Specific Regression Function Classes}
\label{sec:eg_oreg}

\begin{enumerate}
	\item Any finite regression class $\cF$ such that $|\cF| < \infty$, one can choose \oreg\, such that $$\regsq(T) \leq 2 \log |\cF|.$$ 
	\item Class of linear predictors $\cF:= \{(x,a) \mapsto \langle \theta,x_a\rangle \mid \theta \in \R^d, \|\theta\|\le 1\}$. In this case choosing \oreg\, to be the Vovk-Azoury-Warmuth forecaster, as proposed by \citet{vovk98,azoury01}, we have 
	$$\regsq(T) \leq d\log(T/d).$$  
	\item Class of generalized linear predictors $\cF:= \{(x,a) \mapsto \sigma(\langle \theta,x_a\rangle) \mid \theta \in \R^d, \|\theta\|\le 1\}$ where $\sigma:\R \mapsto [0,1]$ is a fixed non-decreasing $1$-Lipschitz link function. Here using GLMtron algorithm of \citet{Sham+11} as \oreg\, leads to $\regsq(T) \leq \sqrt T$. Alternatively a second order variant of GLMtron
	leads to an instance dependent guarantee $\regsq(T) \leq O(d \log T/\kappa_\sigma^2)$, further assuming a lower bound $\kappa_\sigma$ on the gradient of $\sigma$, more precisely $\sigma' \geq \kappa_\sigma > 0$.
	\item Reproducing kernel hilbert space (RKHS) $\cF:= \{f \mid \|f\|_{\cH} \le 1, \cK(x_a,x_a) \le 1\}$: Using (kernelized) Online Gradient Descent, one can obtain $\regsq(T) \le O(\sqrt T)$. 
	\item Banach Spaces $\cF:= \{(x,a) \mapsto \langle\theta,x_a \rangle \mid \theta \in \cB, \|\theta\| \le 1\}$, where $(\cB,\|\cdot\|)$ is a separable Banach space and $x$ belongs to the dual space $(\cB,\|\cdot\|_{*})$: For this setting, whenever $\cB$ is $(2,D)$-uniformly convex, using `Online Mirror Descent' (for example see \citet{orabona15}) as \oreg\, can be configured to have $\regsq(T) \leq (T/D)^{1/2}$ \citep{srebro11}.
\end{enumerate}

\subsection{An Example for Remark~\ref{rem:imposs}}
\label{app:imposs}

%% 	One potential question could be is it necessary to allow the learner randomize the actions ($\{(a_t,b_t)\}_{t \in [T]}$), what if they are not allowed to? Note for such a fixed sequence of actions $\{(a_t,b_t)_{t \in [T]}\}$, the best-response regret becomes:  

Consider the scenario where the leaner is not allowed to randomize and evaluated on a fixed sequence of actions $\{(a_t,b_t)_{t \in [T]}\}$ as defined in \eqref{eq:br_reg_bad}.
The claim is for the regret definition in~\pref{eq:br_reg_bad} it is impossible for the learner to achieve $o(T)$ best-response regret in the worst case. To see why, consider the following counter example with a single context:
\begin{align*}
		f^\star := \rbr{\begin{matrix}
				0 & 1 & -1\\
				-1 & 0 & 1\\
				1 & -1 & 0
 		\end{matrix}}
\end{align*}
 	This matrix is skew symmetric and hence zero-sum. However, for any choice $(a,b)$ of the learner, the adversary has a choice that can	guarantee value of $1/2$. Specifically, if learner chooses $(1,2)$	then adversary chooses $1$, if learner chooses $(1,3)$ then adversary chooses $3$ and if learner chooses $(2,3)$ then adversary chooses	$2$. This shows that we must allow the learner to randomize. 	
\section{Appendix for Sec. \ref{sec:linf}}	

\subsection{\algc: Algorithm Pseudocode for Standard “Non-Contextual” $K$-armed Dueling Bandits}
\label{app:stdb_pseudocode}

\begin{center}
	\begin{algorithm}[h]
		\caption{\textbf{\algc}}
		\label{alg:stdb}
		\begin{algorithmic}[1]	
			\STATE {\bfseries input:} Arm set: $[K]$, parameters $\delta \in (0,1)$
			\STATE {\bfseries init:} $W_{1}[i,j] \leftarrow 0$,  $\forall i,j \in [K]$. Use $N_{t}[i,j] := W_t[i,j] + W_t[j,i] \, \forall t \in [T]$ 
			\FOR{$t = 1, 2, \ldots, T$}
			\STATE $\tilde P_{t}[i,j] := \frac{W_t[i,j]}{N_t[i,j]}, \, \hP_t[i,j]:= 2\tilde P_t[i,j]-1$, $C_t[i,j] \leftarrow \sqrt{\frac{\log{(K^2 t^2/\delta)}}{N_t[i,j]} }, ~\forall i,j \in [K]$ ~\big(we assume $\nicefrac{x}{0}:=0.5, ~\forall x \in \R$\big)
			
			\STATE $U_t[i,j] \leftarrow \hP_t[i,j] + C_t[i,j]$, $U_t[i,j] \leftarrow 0,~~\forall i, j \in [K]$  
			
			%\STATE Suppose $\cC_t:= \{ \u \in \Delta_K \mid  \big[\min_{\e_i \in \Delta_K} \u^\top \U\e_i  > 0 \big] \}$, 
			
			%\STATE Select $(\U_t,\v_t):= \arg\max_{\u, \v \in \cC_t}\u^\top \C^t\v$
			
			\STATE Denoting $p_t^\ell[\cdot]:= \sum_{b = 1}^Kp_t[\cdot,b]$ and $p_t^r[\cdot]:=\sum_{a = 1}^Kp_t[a,\cdot]$, find a policy $\p_t \in \Delta_{K \times K}$ (CCE of $\U_t$) such that:  
			\vspace{-20pt}
			
			\begin{align}
				\label{eq:cce}
				\nonumber & \sum_{a,b \in [K]\times [K]}p_t[a,b] U_t[a,b] \geq \max_{a^\star \in [K]}\Big[\sum_{b \in [K]}p_t^r[b] U_t[a^\star,b]\Big]\\
				&\sum_{a,b \in [K]\times [K]}p_t[a,b] U_t[b,a] \geq \max_{b^\star \in [K]}\Big[\sum_{a \in [K]}p_t^\ell[a] U_t[b^\star,a]\Big],
			\end{align}
			 %respectively denotes the marginal distribution of the left and the right arm for the joint distribution $\p_t$ 
			\vspace{-20pt}
			\STATE Play $(a_t,b_t) \sim \p_t$ %such that $a_t \sim \U_t,\, b_t \sim \v_t$. 
			
			\STATE Receive preference feedback $o_t \in \{-1,1\}$. $\tilde o_t \leftarrow (o_t + 1)/2$  
			\STATE Update $W_{t+1}[a_t, b_t] \leftarrow W_t[a_t, b_t] + \tilde o_t$; $W_{t+1}[b_t, a_t] \leftarrow W_t[b_t, a_t] + (1-\tilde o_t)$
			\ENDFOR 
			%\STATE Return $\bsigma'$
		\end{algorithmic}
	\end{algorithm}
	%\vspace{2pt}
\end{center}

\subsection{Regret Analysis of Algorithm \ref{alg:stdb} }
\label{app:stdb_reg}

\begin{restatable}[Restatement of Thm. \ref{thm:stdb_inf}: Expected regret of \algc\, on \stdb]{thm}{thmstdb}
	\label{thm:stdb}
	For the setting of standard $K$-armed dueling bandit (\stdb), the best-response regret of \algc\, (Alg. \ref{alg:stdb}) satisfies:
\begin{align*}
\breg_T \le O(K\log (KT)\sqrt{T}).
\end{align*} 
\end{restatable}

\begin{proof}
The proof relies on two main results: Confidence bounding $\P$ through $\U(t)$ (Lem. \ref{lem:conf}) and analyzing the instantaneous regret of the column and the row player (Lem. \ref{lem:rowreg}). 
Lem. \ref{lem:conf} simply guarantees that with high probability at least $(1-1/T)$, $\P$ can be sandwithced inside $\boldsymbol{\hP}(t) \pm \C(t)$, which is crucially used in the later part of the proof. 

\begin{restatable}[Confidence Bounding $\P$]{lem}{lemconf} Setting $\delta = \frac{1}{T}$ in Alg. \ref{alg:stdb}, we get
	\label{lem:conf}
	$$Pr\big(  \exists t \in [T], \exists (i,j) \in [K] \times [K],~ \hP_t[i,j] - C_t[i,j] \leq P_t[i,j] \leq \hP_t[i,j] + C_t[i,j] \big) \leq 1/T$$.
\end{restatable}

Recall that $\p_t^{\ell}$ and $\p_t^{r}$ respectively denotes the marginal distribution of the left and right arm of the dueling pair $(a_t,b_t)$ when sampled as $(a_t,b_t) \sim \p_t$.
The next claim upper bounds the regret of both left (row) and the right (column) action at each round $t$. Precisely, Lem. \ref{lem:rowreg} shows the learner's instantaneous regret can be bounded by the expected confidence bounds of played duel as follows:

\begin{restatable}[Learner's Instantaneous Regret]{lem}{lemrowreg}
	\label{lem:rowreg}
For any $\q \in \Delta_K$, and $(a_t,b_t) \sim \p_t$
$$\q^\top \P (\p_t^\ell + \p_t^r) \leq 2\sum_{(a,b) \in [K]\times [K]}p_t[a,b] C_t[a,b].$$
\end{restatable}

The regret bound of Alg. \ref{alg:stdb} (Thm. \ref{thm:stdb}) now follows by summing the instantaneous regret upper bound of Lem. \ref{lem:rowreg} over $t \in [T]$. Precisely,
%\newpage

\begin{align*}
	\breg &= \sum_{t=1}^T \max_{q_t \in \Delta_K} \EE_{a \sim q_t} \EE_{(a_t,b_t) \sim \p_t} \frac{\sbr{P[a, a_t] + P[a,b_t]}}{2}\\
	& \le \sum_{t=1}^T \frac{\bigg[ \max_{\q \in \Delta_K} \q^{\top} \P (\p_t^{\ell} + \p_t^{r} ) \bigg]}{2} \overset{(i)}{\leq} \sum_{t = 1}^T \sum_{(a,b) \in [K]\times [K]}p_t[a,b]C_t[a,b]\\
	& \overset{(ii)}{\leq} \sum_{t = 1}^T C_t[a_t,b_t] + 4\sqrt{T}\log(K T) = 2\sqrt{\ln(K T)}\sum_{t = 1}^T\sqrt{\frac{1}{N_t[a_t,b_t]}} + 4\sqrt{T}\log(K T) \\
	& = 2\sqrt{\ln(K T)} \sum_{a < b} \sum_{\tau = 1}^{N_t[a,b]}\sqrt{ \frac{1}{\tau}} + 4\sqrt{T}\log(K T)\\
	& \overset{(iii)}{\leq} 4\sqrt{\ln(K T)} \sum_{a < b} \sqrt{ N_t[a,b]} + 4\sqrt{T}\log(K T)  \overset{(iv)}{\leq}  4\sqrt{\ln(K T)} \sqrt{K^2T} + 4\sqrt{T}\log(K T) 
\end{align*}
where $(i)$ follows from Lem. \ref{lem:rowreg}, $(ii)$ applies Azuma-Hoeffding's Inequality, (iii) uses $\sum_{i = 1}^n 1/\sqrt{i} \leq  2\sqrt{n}-1$, and (iv) applies Cauchy's Scharwz inequality. This concludes the proof.	
\end{proof}

\subsection{Technical Lemmas for Thm. \ref{thm:stdb}}

\subsubsection{Proof of Lem. \ref{lem:conf}}

\lemconf*

For any $\delta >0$, then, with probability at least $1-\delta$, for any $i,j \in [K]$%
\[%\begin{equation}
\hP_t[i,j]-C_t[i,j] \leq P[i,j] \leq  U_t[i,j] := \hP_t[i,j]+C_t[i,j],  \qquad \forall t \in [T] \,.
\]

\begin{proof}
	Suppose $\cG_t[i,j]$ denotes the event that at time $t \in [T]$ and  item-pair $i,j \in [K]$,  $\hP_t[i,j] - C_t[i,j] \leq P[i,j] \leq \hP_t[i,j] + C_t[i,j]$. 
	Note for any such that pair $(i,i)$, $\cG_t[i,i]$ always holds true for any $t \in [T]$ and $i \in [n]$, as $P_t[i,i] = U_t[i,i] = 0$ by definition. We can thus assume $i \neq j$. Moreover, for any $t$ and $i,j$, $\cG_t[i,j]$ holds if and only if $\cG_t[i,j]$, thus we will restrict our focus only to pairs $i < j$ for the rest of the proof. Hence, to prove the lemma it suffices to show
	\[
	\P \Big( \exists t \in [T], \exists i < j, \text{ such that } ~\cG^c_t[i,j] \Big) \leq \frac{1}{T} \, \,,
	\]
	which we do now.
	$\cG_t[i,j]$ can be rewritten as:
	\[
	|\hP_t[i,j] - P[i,j]| \le C_t[i,j] \,.
	\] 
	Let $\tau_{ij}(n)$ the time step $t \in [T]$ when the pair $(i,j)$ was updated (i.e. $i$ and $j$ was compared) for the $n^{th}$ time. 
	%
	%Clearly for any $n \in \N$, $\tau_{ij}(n+1) \ge \tau_{ij}(n)$ and if $\cG_{ij}$ holds at time $\tau_{ij}(n)$, it should hold at time $\tau_{ij}(n+1)$ as well, just by the definition of $C_t[i,j]$. 
	%
	We now bound the probability of the confidence bound $(\cG_t[i,j])$ getting violated at any round $t \in [T]$ for some duel $(i,j)$ as follows:
	\begin{align*}
		\P \Big( & \exists t \in [T],i<j, \text{ such that } ~\cG^c_t[i,j] \Big) \\
		& \le \sum_{i < j}  \P \Bigg( \exists n \ge 0, |P[i,j] - \hP_{\tau_{ij}(n)}[i,j] | > C_t[i,j] \Bigg) \\
		& = \sum_{i < j}  \P \Bigg( \exists n \in [0,T], ~|P[i,j] - \hP_n[i,j] | > C_t[i,j] \Bigg),
	\end{align*}
	where $\hP_t[i,j] = \frac{W_t[i,j]}{W_t[i,j] + W_t[i,j]}$ is the frequentist estimate of $p[i,j]$ at round $t$ (after $n = N_t[i,j] \in [0,T]$ comparisons between arm $i$ and $j$). 
	Noting $N_{\tau_{ij}(n)}[i,j]= n$, $\tau_{ij}(n) > n$, and using Hoeffding's inequality, we further get
	\begin{align*}
		\P \Big( \exists t \in [T], & i<j, \text{ such that } ~\cG^c_t[i,j] \Big)
		\le \sum_{i < j} \Bigg [ \sum_{n=1}^{T}2e^{-2n\frac{\ln (K^2\tau_{ij}(n)^2/\delta)}{2n}} \Bigg]  \\
		& \leq \sum_{i < j} \Bigg [ \sum_{n=1}^{T} \frac{\delta}{K^2n^2}\Bigg] < \sum_{i < j} \Bigg [ \sum_{n=1}^{\infty} \frac{\delta}{K^2n^2}\Bigg] \\
		& < \Bigg [ \frac{K(K-1)}{2}\frac{\delta \pi^2}{K^2 6}\Bigg]  < \delta = \frac{1}{T}.
	\end{align*}
	This concludes the claim.
\end{proof}

\subsubsection{Proof of Lem. \ref{lem:rowreg}}

\lemrowreg*

\begin{proof}
	Note for any $\q \in \Delta_K$,
	\begin{align*}
		&\q^{\top}\P \p_t^\ell = \sum_{a^\star = 1}^{K}\sum_{a = 1}^{K}q[a^\star]p_t^\ell[a] P[a^\star,a] \leq \sum_{a^\star = 1}^{K}\sum_{a = 1}^{K}q[a^\star]p_t^\ell[a] U_t[a^\star,a]\\
		& = \max_{a^\star \in [K]}\sum_{a = 1}^{K}p_t^\ell[a] U_t[a^\star,a] \le \sum_{a,b \in [K]\times [K]}p_t[a,b] U_t[a,b] 
	\end{align*}
where the first inequality follows from Lem. \ref{lem:conf} and last inequality by the second inequality constraint of the CCE equations (see Eqn. \eqref{eq:cce}).% The last inequality uses $U[i,i] = 0, ~\forall i \in [K]$ by definition.
%holds assuming $p_{a,b} = p_{b,a}$ \red{symmetry is required} (noting $U_{ba} = -U_{ab}, ~\forall a,b \in [K]$).

On the other hand for the right action (column player), similarly again for any $\q \in \Delta_K$:
	\begin{align*}
		&\q^\top \P \p_t^r = \sum_{a^\star = 1}^{K}\sum_{a = 1}^{K}q[a^\star]p_t^r[a] P[a^\star,a] \overset{\text{Lem.} \ref{lem:conf}}{\leq} \sum_{a^\star = 1}^{K}\sum_{a = 1}^{K}q[a^\star]p_t^r[a] U_t[a^\star,a]\\
		& = \max_{a^\star \in [K]}\sum_{a = 1}^{K}p_t^r[a] U_t[a^\star,a] \le \sum_{a,b \in [K]\times [K]}p_t[a,b] U_t[a,b] 
		\end{align*}
where the last inequality follows from first inequality constraint of the CCE equations (see Eqn. \eqref{eq:cce}).% and the last inequality uses $U[i,i] = 0, ~\forall i \in [K]$.

Finally combining above two results and noting that for any $(a,b) \in [K]\times [K]$, $U_t[b,a] = \hP_t[b,a] + C_t[b,a] = -(\hP_t[a,b]+C_t[a,b]) + C_t[a,b] + C_t[b,a] = - U_t[a,b] + 2C_t[a,b]$, the claim follows.
\end{proof}

\subsection{Other Structured Function Classes}
\label{sec:linf_o}

\algc~ (Alg. \ref{alg:stdb}), analyzed above, can be extended to other parametric structured function classes as well which may support a statistical estimation based techniques, such as generalized linear function classes, etc. We briefly discuss the case for linear function classes here: 

\vspace{3pt}\noindent 
\textbf{Setting: Dueling-Bandits with Linear Realizability (\textrm{LinDB}$(d)$)} Consider $\cX \subseteq [-1,1]^{K\times K \times d}$, such that if $\x_t$ is the context received at time $t$ then  $f(x_t)[a,b]:= \w^\top \x_t[a,b] \in [-1,1]$ for any pair $(a,b) \in [K] \times [K]$, for some unknown $\w \in [-1,1]^d$. Note $f(\x) \in [-1,1]^{K \times K}$ for any $\x \in \cX$. Considering the same setup of Sec. \ref{sec:prob}, the goal is to again minimize the $\breg$. We detail the complete pseudocode in Alg. \ref{alg:linf} (\algl), and analyze its regret guarantee as follows:
 
\begin{restatable}[Expected regret for \textrm{LinDB}]{thm}{thmlinf}
	\label{thm:linf}
	For the setting of dueling bandits with linear-realizability \textrm{LinDB}$(d)$ class, we have
	$\breg_T \le O(d\log (KT)\sqrt{T}).$ %\mathrm{LinDB}\mbox{-}\mathrm{BR}\mbox{-}\mathrm{Regret}(f^\star)
\end{restatable}

%\subsection{Algorithm for \textrm{LinDB}$(d)$}
%\label{app:linf} 

%\vspace{-8pt}
\begin{algorithm}[h]
	\caption{\textbf{\algl}}
	\label{alg:linf}
	\begin{algorithmic}[1]
		\STATE {\bfseries input (tuning parameters):} Regularizer $\lambda >0$, Exploration length $t_0>0$% (to be optimized), and learning rate $t_0 > 0$
		\STATE {\bfseries init:} Set $V_{1} := \lambda \I_d $, $Y_1 \leftarrow 0$, $X_1 \leftarrow \0_d \in \R^{d}$ 
		\FOR{$t = 1,2,\ldots T$}
		%    \STATE Compute the MLE estimate $\htheta_t$ which satisfies:
		{
			\STATE $\hat \w \leftarrow (X_t^\top X_t + \lambda \I)^{-1}X_t^\top Y_t$ 
		}%small
		\STATE Receive the context vector $\x_t \in [-1,1]^{K \times K \times d}$
		\STATE Pairwise preference estimates: $\hP_t[a,b] \leftarrow \hat \w ^\top \x_t[a,b]$, for all $[a,b] \in [K] \times [K]$ 
		\STATE UCB estimates: $U_t[a,b] \leftarrow \hP_t[a,b] + \eta \sqrt{\x_t[a,b]^\top V_t^{-1} \x_t[a,b]}$. $U_t(a,a) \leftarrow 0$, for all $(a,b) \in [K] \times [K]$
		\STATE Find a policy $\p_t \in \Delta_{K \times K}$ (CCE of $U(t)$) such that:  
			\begin{align*}
			\nonumber & \sum_{a,b \in [K]\times [K]}p_t[a,b] U_t[a,b] \geq \max_{a^\star \in [K]}\Big[\sum_{b \in [K]}p_t^r[b] U_t[a^\star,b]\Big]\\
			&\sum_{a,b \in [K]\times [K]}p_t[a,b] U_t[b,a] \geq \max_{b^\star \in [K]}\Big[\sum_{a \in [K]}p_t^\ell[a] U_t[b^\star,a]\Big],
		\end{align*}
	\quad\quad\quad\quad\quad\quad\quad\quad\quad  where $p_t^\ell[\cdot]:= \sum_{b = 1}^Kp_t[\cdot,b]$ and $p_t^r[\cdot]:=\sum_{a = 1}^Kp_t[a,\cdot]$.
		\STATE Play $(a_t,b_t) \sim \p_t$. 
		\STATE Receive preference feedback $o_t \in \{-1,1\}$. 
		\STATE Update $V_{t+1} = V_t + \x_t[a_t,b_t]\x_t[a_t,b_t]^\top \in \R^{d \times d}$
		\STATE Update $Y_{t+1} \leftarrow [Y_t; o_t] \in \R^{t+1}$, $X_{t+1} \leftarrow [X_t; \x_t[a_t,b_t]] \in \R^{t+1 \times d}$
		\ENDFOR
	\end{algorithmic}
\end{algorithm}
%\vspace{-10pt}

Thm. \ref{thm:linf} gives the $\breg$\, of \algl\, (Alg. \ref{alg:linf} for the \textrm{LinDB} setup). The regret analysis of Thm. \ref{thm:linf} follows exactly same as the proof of Thm. \ref{thm:stdb} along with applying the standard concentration techniques from the linear bandits literature with proper tuning of $\lambda$ and $\eta$ (see \cite{Yadkori11,li17,CsabaNotes18} for details on concentration results of linear bandits).

\section{Appendix for Sec. \ref{sec:genf}}	

%%%%%%%%%%%%%%%%%%%%%%%%%%%%%%%%%%%%%

\subsection{Proof of Thm. \ref{thm:minmax}}
\label{app:main_thm_proof}

\begin{proof}
	Start by noting that when the learner plays $(a_t,b_t)$ from a product distribution s.t. $(a_t,b_t) \sim \p_t \times \p_t$. Then the best-response regret becomes:% can be rewritten as:
	\begin{align*}
		\breg_T &:= \sum_{t=1}^T \max_{\q \in \Delta_K} \EE_{a \sim \q} \EE_{a_t \sim \p_t} \sbr{f^\star(x_t)[a, a_t]} = \sum_{t=1}^T \max_{\q \in \Delta_K} \q^\top f^\star \p_t\\
		& \leq \Big[ \frac{\gamma\regsq(T)}{4} + \sum_{t=1}^T \frac{5K}{\gamma}  \Big] = \sqrt{5 KT \regsq(T)}
	\end{align*} 
	where the last inequality follows from Lem. \ref{lem:minmax}, and last equality is due to choosing $\gamma = \sqrt{\frac{20 KT}{\regsq(T)}}$, which leads to the desired regret guarantee of Thm. \ref{thm:minmax}. Further, note we need the constraint $T \geq
	4K\regsq(T)$ since Lem. \ref{lem:minmax} requires $\gamma > 2K$. And since we set $\gamma =  \sqrt{\frac{20 KT}{\regsq(T)}}$, this is satisfied only if $T \geq
	4K\regsq(T)$.  
\end{proof}

%%%%%%%%%%%%%%%%%%%%%%%%%%%%%%%%%%%%%%%%%%%%%%%%%%%%%%%%%%%%%%%%%%%%%%%%%%%

\subsection{Simpler Analysis of Thm. \ref{thm:minmax} (using Eqn. \ref{eq:Vgam})}
\label{app:simpler_calculation}

Given the choice of $\p_t$ as shown in Eqn. \ref{eq:Vgam}, we now give  {simpler and more direct proof} of Thm. \ref{thm:minmax}. %(or more precisely Lem. \ref{lem:minmax} from which Thm. \ref{thm:minmax} follows straightforwardly). 
%
%Let us first try to prove the per-round inequality of Lem. \ref{lem:minmax} more intuitively. 
Recall we assume the that dueling arms $(a_t,b_t)$ are drawn from a product measure $(a_t,b_t) \sim \p_t \times \p_t$ for some $\p_t \in
\Delta_K$ at each round. %Then rewriting the LHS of Eqn. \eqref{eq:Vgam} in matrix-vector notation:
Now suppose we find $\p_t$ that satisfies Eqn. \eqref{eq:Vgam}.

%\newpage

Then for any $f^\star$ and any $\q$, using Eqn. \eqref{eq:Vgam} we have:
\begin{align*}
	\q^\top f^\star \p_t &\leq \q^\top (f^\star - \hat{Y}_t) \p_t - \frac{2}{\gamma} \sum_a \frac{q[a]}{p_t[a]} + \frac{5K}{\gamma}\\
	& = \sum_{a=1}^K \frac{q[a]}{\sqrt{p_t[a] \gamma}} \cdot \sqrt{p_t[a]\gamma} (f^\star[a,\cdot] - \hat{Y}_t[a,\cdot]) \p_t - \frac{2}{\gamma} \sum_a \frac{q[a]}{p_t[a]} + \frac{5K}{\gamma}\\
	& = \sum_{a=1}^K \sqrt{\frac{q[a]^2}{{p_t[a] \gamma}} \cdot {p_t[a]\gamma} \big((f^\star[a,\cdot] - \hat{Y}_t[a,\cdot]) \p_t\big)^2} - \frac{2}{\gamma} \sum_a \frac{q[a]}{p_t[a]} + \frac{5K}{\gamma}\\
	& \leq \frac{1}{2\gamma} \sum_a \frac{q[a]^2}{p_t[a]} + \frac{\gamma}{2} \sum_a p_t[a] \rbr{ (f^\star[a,\cdot] - \hat{Y}_t[a,\cdot]) \p_t}^2 - \frac{2}{\gamma} \sum_a \frac{q[a]}{p_t[a]} + \frac{5K}{\gamma}\\
	& \leq \frac{2}{\gamma} \sum_a \frac{q[a]}{p_t[a]} + \frac{\gamma}{2} \sum_a p_t[a] \rbr{ (f^\star[a,\cdot] - \hat{Y}_t[a,\cdot]) \p_t}^2 - \frac{2}{\gamma} \sum_a \frac{q[a]}{p_t[a]} + \frac{5K}{\gamma}\\
	& = \frac{\gamma}{2} \sum_a p_t[a] \rbr{ (f^\star[a,\cdot] - \hat{Y}_t[a,\cdot]) \p_t}^2 + \frac{5K}{\gamma}\\
	& = \frac{\gamma}{2} \sum_a p_t[a] \rbr{ \EE_{b \sim \p_t} (f^\star[a,b] - \hat{Y}_t[a,b])}^2 + \frac{5K}{\gamma}\\
	& \leq \frac{\gamma}{2} \sum_{a,b} p_t[a]p_t[b] (f^\star[a,b] - \hat{Y}_t[a,b])^2 + \frac{5K}{\gamma},
\end{align*}
where the first inequality follows from Eqn. \eqref{eq:Vgam}, the second inequality uses AM-GM inequality: For any $x,y \in \R_+, \frac{x+y}{2} > \sqrt{xy}$, and lastly we use $q[a]^2 < q[a]$ in the third inequality. 

Then proceeding similar to the proof of Thm. \ref{thm:minmax}, the best-response regret of any learner playing $(a_t,b_t) \sim \p_t \times \p_t$ becomes:% can be rewritten as:
\begin{align*}
	\breg_T &:= \sum_{t=1}^T \max_{\q \in \Delta_K} \EE_{a \sim \q} \EE_{a_t \sim \p_t} \sbr{f^\star(x_t)[a, a_t]} = \sum_{t=1}^T \max_{\q \in \Delta_K} \q^\top f^\star \p_t\\
	& \leq \Big[ \frac{\gamma\regsq(T)}{2} + \sum_{t=1}^T \frac{5K}{\gamma}  \Big] = \sqrt{10 KT \regsq(T)}
\end{align*} 
setting $\gamma = \sqrt{\frac{10KT}{\regsq}}$ following the similar line of argument as shown in the proof of Thm. \ref{thm:minmax}. This yields the desired $O(\sqrt{KT \regsq})$ regret guarantee of Alg. \ref{alg:genf}. 

%%%%%%%%%%%%%%%%%%%%%%%%%%%%%%%%%%%%%%%%%%%%%%%%%%%%%%%%%%%%%%%%%%%%%%%%%%%

\subsection{Regret Bound of Alg. \ref{alg:genf} for some special realizability function classes}
\label{app:spcl_reg} 

	We analyze some special function classes and derive the $\breg$ \, of Alg. \ref{alg:genf} for each cases using the specific regression oracles (recall the details and notations from Rem. \ref{rem:cases} and Appendix \ref{sec:eg_oreg}): 
	
	\begin{enumerate}
	\item Any finite regression class $\cF$ such that $|\cF| < \infty$: Since we can have \oreg\, oracle such that $\regsq(T) \leq 2 \log |\cF|$, this yields $\breg_T[\algf] = O(\sqrt{KT \log |\cF|})$. For example in the standard $K$-armed (non-contextual) dueling bandit setting, one can construct $\cF$ such that $\log |\cF| = O\big(K^2\log (KT)\big)$, and hence implying the $\breg(T)$ of Alg. \ref{alg:genf} to be $O\big(\sqrt{K^3T\log (KT)}\big)$ in this case, which is though $O(K)$ multiplicative factor worse than the optimal rate \cite{CDB}.
	\item Class of linear predictors $\cF:= \{(x,a) \mapsto \langle \theta,x_a\rangle \mid \theta \in \R^d, \|\theta\|\le 1\}$. Here we can have \oreg\, oracle with $\regsq(T) \leq d\log(T/d)$, which yields $\breg_T[\algf] = O(\sqrt{d K T \log(T/d)})$.   
	\item Alternatively, for linear predictors, instantiating $\oreg$ as online gradient descent guarantees $\regsq(T) \leq O(\sqrt{T})$ and hence $\breg(T) \leq O(K^{1/2} T^{3/4})$. 
	\item Class of generalized linear predictors $\cF:= \{(x,a) \mapsto \sigma(\langle \theta,x_a\rangle) \mid \theta \in \R^d, \|\theta\|\le 1\}$. Since we can have \oreg\, oracle such that $\regsq(T) \leq d\log T/ \kappa_\sigma^2$, one can achieve $\breg_T[\algf] = O\Big(\sqrt{\frac{d K T \log T}{\kappa_\sigma^2}}\Big)$.   
	\item Reproducing Kernel Hilbert Space (RKHS) $\cF:= \{f \mid \|f\|_{\cH} \le 1, \cK(x_a,x_a) \le 1\}$: Since we can have \oreg\, oracle such that $\regsq(T) \leq O(\sqrt T)$, one can achieve\\ $\breg_T[\algf] = O(\sqrt{KT^{3/4}})$.
	\item Banach Spaces $\cF:= \{(x,a) \mapsto \langle\theta,x_a \rangle \mid \theta \in \cB, \|\theta\| \le 1\}$. Since we can have \oreg\, oracle such that $\regsq(T) \leq (T/D)^{1/2}$, one can achieve $\breg_T[\algf] = $ $O(\sqrt{KD^{-1/2}T^{3/4}})$.
	\end{enumerate}
	
}

\end{document}